\documentclass[twoside,11pt]{article}

\usepackage{blindtext}

\usepackage[abbrvbib, preprint]{jmlr2e}

\usepackage{amsmath}
\usepackage{mathrsfs}
\usepackage{shuffle}
\usepackage{dsfont}
\usepackage{cleveref}
\usepackage{stmaryrd}
\usepackage{booktabs}
\usepackage{multirow}
\usepackage{makecell}
\usepackage{graphicx}
\usepackage{subcaption}
\usepackage[acronym]{glossaries}
\usepackage[inkscapelatex=false]{svg}
\numberwithin{equation}{section} 
\usepackage{minted}
\usepackage{listings}
\newacronym[plural=one-class SVMs,longplural=one-class support vector machines]{ocsvm}{one-class SVM}{one-class support vector machine}
\newacronym[plural=RKHSs,longplural=reproducing kernel Hilbert spaces]{rkhs}{RKHS}{reproducing kernel Hilbert space}
\newacronym[]{cvar}{CVaR}{conditional value-at-risk}

\usepackage{mathtools}

\Crefformat{figure}{#2Fig.~#1#3}
\Crefmultiformat{figure}{Figs.~#2#1#3}{ and~#2#1#3}{, #2#1#3}{ and~#2#1#3}

\usepackage[textsize=tiny]{todonotes}
\numberwithin{theorem}{section}

\newcommand{\R}{\mathbb{R}}
\newcommand{\X}{\mathcal{X}}
\newcommand{\N}{\mathbb{N}}

\newcommand{\pr}{\mathbb{P}}
\newcommand{\ex}{\mathbb{E}}

\newcommand{\h}{\mathcal{H}}
\newcommand{\Hh}{\mathbb{H}}
\newcommand{\Oo}{\mathcal{O}}
\newcommand{\Cc}{\mathscr{C}}

\newcommand{\statespace}{\mathbb{R}^d}
\newcommand{\tensora}{A}

\newcommand{\Hent}{\boldsymbol{\mathrm{H}}}

\newcommand{\id}{\mathbf{1}}
\newcommand{\rkhsnorm}[1]{\|#1\|_{2}}

\newtheorem{assumption}[theorem]{Assumption}

\makeatletter
\renewenvironment{keywords}{%
\par\noindent{\bf Keywords:}\ }%
{\par\vskip0pt} 
\makeatother

\usepackage{lastpage}
\jmlrheading{23}{2025}{1-\pageref{LastPage}}{NA}{NA}{NA}{Ioannis Gasteratos, Antoine Jacquier, Maud Lemercier, Terry Lyons, Cristopher Salvi}

\ShortHeadings{Novelty detection on path space}{}
\firstpageno{1}

\begin{document}

\title{Novelty detection on path space}

\author{\name Ioannis Gasteratos 
        \email i.gasteratos@tu-berlin.de \\
        \addr Institute of Mathematics\\
        TU Berlin
       \AND
       \name Antoine Jacquier 
       \email a.jacquier@imperial.ac.uk \\
       \addr Department of Mathematics\\
       Imperial College London
       \AND
       \name Maud Lemercier 
       \email maud.lemercier@maths.ox.ac.uk \\
       \addr Mathematical Institute \\
       University of Oxford
       \AND
       \name Terry Lyons 
       \email terry.lyons@maths.ox.ac.uk \\
       \addr Mathematical Institute \\
       University of Oxford
       \AND
       \name Cristopher Salvi 
       \email c.salvi@imperial.ac.uk \\
       \addr Department of Mathematics\\
       Imperial College London
       }
\editor{My editor}

\maketitle

\begin{abstract}
We frame novelty detection on path space as a hypothesis testing problem with signature-based test statistics. Using transportation-cost inequalities of~\citet{gasteratos2023transportation}, we obtain tail bounds for false positive rates that extend beyond Gaussian measures to laws of RDE solutions with smooth bounded vector fields, yielding estimates of quantiles and p-values. Exploiting the shuffle product, we derive exact formulae for smooth surrogates of conditional value-at-risk (CVaR) in terms of expected signatures, leading to new one-class SVM algorithms optimising smooth CVaR objectives. We then establish lower bounds on type-$\mathrm{II}$ error for alternatives with finite first moment, giving general power bounds when the reference measure and the alternative are absolutely continuous with respect to each other. Finally, we evaluate numerically the type-$\mathrm{I}$ error and statistical power of signature-based test statistic, using synthetic anomalous diffusion data and real-world molecular biology data.
\end{abstract}

\begin{msc}
 60L10, 60L20, 62M07   
\end{msc}  
\tableofcontents

\section{Introduction}

\medskip \noindent
Determining whether an observed trajectory is consistent with a reference process is a fundamental novelty detection problem, with applications ranging from detecting unusual particle motion in biological cells~\citep{munoz2021unsupervised} to identifying rare astronomical signals in telescope data~\citep{wilensky2019absolving} and detecting malicious activities in network traffic~\citep{cochrane2021sk}. Most existing approaches cast this as one-class classification on time series data and focus primarily on feature engineering or representation learning to obtain suitable embeddings \citep{li2023deep,zamanzadeh2024deep}. In this paper, we frame trajectory-based novelty detection as a hypothesis testing problem on path space. The goal is to test whether an observed path $X:[0,T]\to\R^d$ is drawn from a reference measure $\mu$ representing normal behaviour, or from an alternative distribution $\nu\neq\mu$, i.e. we want to test the null hypothesis
\begin{align*}
    H_0:X\sim \mu \text{ against } H_1:X\sim \nu\neq \mu,
\end{align*} 
where $\nu$ and $\mu$ are two distributions on some space of paths $\X\subset C([0,T];\R^d), T>0$. To perform this hypothesis test, we map each sample path $X$ to a scalar statistic via a map $f:\X\to\R$ which takes the form $f(X)=\langle w, S(X)\rangle$ where $w\in T((\R^d)^*)$ is a linear functional acting on the signature of the process, given by the formal tensor
series of iterated integrals
\begin{align*}
    S(X)=1+\sum_{k=1}^{\infty}S^k(X)\in T((\R^d)),
    \text{ with } S^k(X)=\int_{0<t_1<\cdots<t_k<T} dX_{t_1}\otimes\cdots \otimes dX_{t_k},
\end{align*}
and $\langle\cdot,\cdot\rangle$ denotes the natural pairing between the infinite formal series $T((\R^d))=\prod_{k=0}^{\infty} (\R^d)^{\otimes k}$ and the finite sequences in $T((\R^d)^*)=\bigoplus_{k=0}^{\infty}((\R^d)^{\otimes k})^*$. The null hypothesis $H_0$ is rejected if the observed path $X$ belongs to the rejection region $\Omega_r$ of the form
\begin{equation}\label{rejection_region}
    X\in \Omega_r\iff \langle w, S(X)\rangle>r.
\end{equation}
This framework entails two types of errors. If $X\sim \mu$, but $H_0$ is rejected based on (\ref{rejection_region}), the test incurs a type-$\mathrm{I}$ error. The probability of making such an error---the false positive rate---is given by $\alpha:=\mu(x\in \Omega_r)$. Conversely, if $X\sim \nu\neq \mu$, but $H_0$ is not rejected based on (\ref{rejection_region}), a type-$\mathrm{II}$ error occurs, with probability $\beta:=\nu(x\not\in \Omega_r)$. The power of our test against the alternative $\nu$ is given by $1-\beta$.

The central challenge is to choose the linear functional $w$ and the threshold $r$ to achieve high power against a broad class of alternative distributions $\nu$ on $\X$, while maintaining a fixed significance level $\alpha\in[0,1]$ (i.e., controlling the probability of type-$\mathrm{I}$ error). 
In statistical hypothesis testing, the test can equivalently be formulated in terms of p-values. Specifically, it is common to report the probability of obtaining test results at least as extreme as the result actually observed, as a measure of significance of the test (here, evidence of path $X$ being a novelty) rather than whether or not the null hypothesis was rejected at a pre-determined level of significance $\alpha$. 

Regarding critical value selection, the threshold $r$ corresponding to a given significance level $\alpha$ is determined by $r(\alpha)=Q_{f(X)}(1-\alpha)$, the $(1-\alpha)$-quantile of the null distribution of the test statistic $f(X)$. The $\alpha$-quantile of a real-valued random variable $Z$ with law $P_Z$ is defined as 
\begin{equation*}
    Q_{Z}(\alpha):=\inf\{r \in \R \mid P_Z(z\leq r)\geq \alpha\}.
\end{equation*}
In risk management contexts, quantiles are also known as values at risk. In the sequel they will be denoted by VaR$_\alpha(Z).$
If the cumulative distribution function $F$ of the random variable $f(X)$ under the null hypothesis is continuous, this reduces to the relation $r(\alpha)=F^{-1}(1-\alpha)$. However, in general, for a given measure $\mu$ on $\X$ and test statistic $f:\X\to\R$ the distribution of $f(X)$ is not available in closed form. Therefore, it is rarely possible to conduct the above test. A common workaround is to use a calibration sample $X^{(1)}, \ldots, X^{(n)} \overset{\text{i.i.d.}}{\sim} \mu$ and the corresponding empirical estimator $F_n(f(x))=\frac{1}{n}\sum_{i=1}^{n}\mathds{1}\{f(X^{(i)})\leq f(x)\}$ to determine the critical value of the test. While simple and widely used, this approach comes with important caveats~\citep{bates2023testing}. Conditioned on a fixed calibration set, the false positive rate may exceed the target level $\alpha$; it is only guaranteed to be controlled marginally over calibration sets. Moreover, when empirical p-values are used, the dependencies introduced by sharing the same calibration set can invalidate certain multiple testing procedures, such as Fisher’s combination test. These limitations are especially relevant in novelty detection applications on data streams, where practitioners must often test multiple data segments sequentially. For instance, in radio astronomy, the task of identifying radio-frequency interference (RFI) in telescope visibility measurements involves repeatedly testing thousands of short time-frequency windows for anomalies~\citep{arrubarrena2024novelty}. A similar challenge arises in biological signal analysis, where electrical readouts from RNA molecules are scanned segment by segment to identify potential chemical modifications~\citep{leger2021rna}. A solution is to derive an upper bound on the tail probabilities of the test statistic, ensuring that the rejection threshold controls the false positive rate regardless of the calibration sample, or with high probability over the sample.

Regarding the choice of score function (one-class classifier) $f:\X\to\R$, it is important to ensure that the test has high statistical power, i.e. is effective at flagging outliers. Since the signature has the \emph{universal approximation property} on compact subsets of $\X$ (the set of functions of the form $f(x)=\langle w, S(x)\rangle$ is dense in the set of continuous functions on any such compact subset), the class of test statistics we consider is extremely rich \citep{arribas2018derivatives}. Leveraging this universal approximation property and other algebraic features, signature methods have rapidly gained traction across domains such as quantitative finance~\citep{perez2020signatures, cuchiero2023signature, abi2025signature, cuchiero2023signature, bonesini2024rough}, cybersecurity~\citep{cochrane2021sk}, information theory~\citep{salvi2021rough, salvi2023structure, shmelev2024sparse}, and quantum computing~\citep{crew2025quantum}. They underpin universality results for \emph{neural differential equations}~\citep{morrill2021neural, arribas2020sigsdes} and \emph{state-space models} (SSMs)~\citep{cirone2024theoretical, cirone2025parallelflow, walker2025structured}. In generative modelling, signatures have been used to synthesise financial time series~\citep{buehler2020data}, act as universal nonlinearities in Seq2Seq models~\citep{kidger2019deep}, and provide representation spaces for training \emph{score-based diffusion models}~\citep{barancikovasigdiffusions}. 
An introduction to signature methods in machine learning can be found in~\citep{cass2024lecture} and a survey of recent applications in~\citep{fermanian2023new}.
Various Python packages for signature computations are readily available, such as \texttt{esig}~\citep{esig}, \texttt{iisignature}~\citep{reizenstein2018iisignature}, \texttt{signatory}~\citep{signatory} and more recently \texttt{pySigLib}~\citep{shmelev2025pysiglib}. In this paper we made use of \texttt{pySigLib} because of its superior performance compared to the other alternatives.

\subsection{Contributions}

We leverage transportation-cost inequalities recently established by~\citet{gasteratos2023transportation} to derive an upper bound on the false positive rate $P_{\mathrm{I}}(r):=\mu(f(x)> r)$ for signature-based test statistics on path space. These tail estimates extend beyond Gaussian measures to a much broader class of reference measures $\mu$, including laws of solutions of RDEs with Gaussian drivers with bounded and sufficiently smooth vector fields. Our tail estimates yield theoretical upper bounds on the critical value $r(\alpha)$ for a desired level of significance~$\alpha$. When numerically tractable, these provide valid (or super-uniform) p-values for hypothesis tests, though potentially being more conservative and decreasing the power.

Beyond hypothesis testing, extreme events and tails of random variables are central in problems of risk management. For example, in finance, quantiles of loss distributions serve as value-at-risk (VaR) measures, while the \acrfull{cvar} at a given confidence level captures the expected loss given that the loss exceeds the VaR at that level. \acrshort{cvar} admits a dual formulation as the solution to a minimisation problem involving the max-function, which is non-differentiable. In this paper, we consider smooth surrogates where the max-function is replaced by a polynomial. Leveraging the shuffle product property of the signature, we show that for any such smooth surrogate, when $f(X)=\langle w, S(X)\rangle$ is our signature-based test statistic for the measure $\mu$, the resulting smooth \acrshort{cvar} admits an exact analytic expression in terms of the expected signature of $\mu$, whenever $\ex_\mu[|f(X)|]<\infty$. This yields new \acrlong{ocsvm} algorithms for novelty detection that directly optimise smooth CVaR objectives over the space of signature-based statistics. 

We derive lower bounds on the type-$\mathrm{II}$ error probability $P_{\mathrm{II}}(r)=\nu(f(x)\geq r)$, for any alternative $\nu$ that has finite first moment. These provide informative upper bounds on the power of the test against any alternative that has finite relative entropy $\Hent(\nu|\mu)$. 
While the transportation-cost inequalities from~\citet{gasteratos2023transportation} imply well-known tail estimates~\citep{cass2013integrability} and the type-$\mathrm{I}$ error bound can be retrieved using these prior results for a specific class of reference measures $\mu$, the lower bound on the type-$\mathrm{II}$ error fundamentally requires the new results in~\citep{gasteratos2023transportation}.

\subsection{Related work}

There exists a large body of machine learning methods for anomaly detection on time-series data. One prominent line of work has focused on extending well-established \emph{algorithms} for anomaly detection on vectorial data via functional or signature embeddings. For instance, Functional Isolation Forest (IF)~\citep{staerman2019functional} adapts Isolation Forest to function spaces, and Signature IF~\citep{staerman2024signature} further leverages rough path signatures to handle multivariate processes; nearest neighbours based methods using signature-derived features have been proposed~\citep{shao2020dimensionless} and One-class SVM has been equipped with multiscale signature features~\citep{mignot2024anomaly}. Another direction  leverages the representation learning capabilities of neural networks. Notably,~\citet{ruff2018deep, ruff2019self} combines deep learning with Support Vector Data Description (SVDD), where neural networks learn feature representations from high-dimensional data such as images or text. 

On the theoretical side,~\citet{cass2024variance} derives the distribution of Mahalanobis distances under Gaussian measures on Hilbert spaces, but a formal hypothesis testing procedure for anomaly detection on path space---especially under non-Gaussian signature embeddings---remains absent. Consequently, most approaches are evaluated using threshold-independent metrics such as the Area Under the Receiver Operating Characteristic curve (AUROC), which measure the ability of a score function to rank anomalies above Normal observations when labelled data is available. While useful for benchmarking, such metrics do not yield a principled framework for making and reporting discoveries at a controlled false positive rate. In this work, we close this gap by establishing explicit bounds on the type-$\mathrm{I}$ error for anomaly detection on path space under weaker distributional assumptions.

A related line of work by~\citet{bates2023testing} establishes confidence bounds on the false positive rate. These bounds hold with high probability over the random draw of the calibration set, providing stronger guarantees on the control of the false positive rate compared to the vanilla conformal inference approach. In contrast, the bounds we establish are not confidence intervals but deterministic tail bounds: they hold with full probability and guarantee type-$\mathrm{I}$ error control uniformly across thresholds. Moreover, the tightness of the upper confidence bound from~\cite{bates2023testing} which is of the form $h(F_n(r))$ where $F_n(r)=\frac{1}{n}\sum_{i=1}^{n}\mathds{1}[f(X^{(i)})>r]$ is limited by the size of the calibration set. 

Our results hold for probability measures satisfying suitable generalised Transportation-Cost Inequalities (TCIs) from~\cite{gasteratos2023transportation}, which imply well-known estimates for the $p$-variation of the law of solutions of RDEs driven by Gaussian rough paths~\citep{cass2013integrability}. Determining the tail behaviour of a variable of interest (here a test statistic for anomaly detection), i.e. the
rate at which the probability of exceeding increasingly high values decreases, is a central question in extreme value theory (EVT). However, in this field, the goal is to describe the statistics of observed extremes using asymptotic approximations. The number of available observations strongly affects the estimation of the shape parameter of the general distributions used in EVT, which ultimately determines the limiting type. Importantly, the Weibull distribution should not be confused with the third limiting type of the EVT, termed the \emph{reversed Weibull}. 

Another hypothesis testing problem that has been studied on path space is the two-sample test~\citep{wynne2022kernel}.~\cite{chevyrev2022signature} introduced a signature-based maximum mean discrepancy (MMD) framework for comparing path distributions, which is used for example in~\cite{andres2024signature} to compare sample paths of fractional Brownian motion. The two-sample setting tests whether two unknown distributions are equal based on samples from each. In this paper, we consider the one-sample goodness-of-fit problem testing whether a single observed path originates from a fixed reference distribution $\mu$. Additionally, while the test may be conducted in practice using samples from $\mu$, our theoretical framework provides new tail and \acrshort{cvar} estimates for the population-level measure $\mu$.

\section{Test statistics on path space}\label{sec:TestStatistics}

We provide a brief overview of commonly used anomaly scores on path space, which are given (or can be approximated) by a linear functional of the signature. These include distance-based score functions as well as solutions of one-class SVM optimisation problems. More details can be found in \Cref{sec:test_statistcs}. 

In the sequel we fix $T>0$ and consider bounded variation paths $C^{1\text{-var}}([0,T];\R^d)$, namely continuous paths $x:[0,T]\rightarrow\R^d$ with bounded total variation on each sub-interval $[s,t]\subset[0,T]$, i.e.
$$
\|x\|_{1\mbox{-var},[s,t]}
:= \sup_{D} \sum_{k}\|x_{t_{k+1}}-x_{t_k}\| < \infty,
$$

where $\|\cdot\|$ is the Euclidean norm and where the supremum is taken over all finite dissections $D = \{s = t_1 < \cdots < t_K = t\}$. We consider test statistics that are linear combination of signature coordinates of the path. The coordinate associated with a word $(i_1,\ldots, i_k)\in\{1,\ldots,d\}^k$ of length $k$ is 
\begin{align*}
    S^{i_1,\ldots,i_k}(x)=\int_{0<t_1<\cdots<t_k<T} dx^{i_1}_{t_1}\cdots dx^{i_k}_{t_k},
\end{align*}
where the integral is to be understood as the Riemann-Stieltjes integral. While the signature is an infinite formal series
\begin{align*}
S(x)=1+\sum_{k=1}^{\infty}\sum_{i_1,\ldots,i_k} S^{i_1,\ldots,i_k}(x)e_{i_1}\cdots e_{i_k}  \in T((\R^d))=\prod_{k=0}^\infty(\R^d)^{\otimes k},
\end{align*}
where $\{e_i:i=1, \ldots, d\}$ is a basis of $\R^d$, the truncated signature at level $N\geq 1$, corresponds to the collection of all signature terms associated to words of length $k \leq N$:
\begin{align}\label{eq:truncated_signature}
S_N(x)=1+\sum_{k=1}^{N}\sum_{i_1,\ldots,i_k} S^{i_1,\ldots,i_k}(x)e_{i_1}\cdots e_{i_k}  \in T^N(\R^d)=\bigoplus_{k=0}^{N}(\R^d)^{\otimes k}.
\end{align}
We consider elements $w$ of $T(\R^d)^*$, the dual space of the tensor algebra $T(\R^d)$. Products of iterated integrals can be re-expressed
as a linear combination of higher-order iterated integrals. More formally, the shuffle identity states that for any two linear functionals $w_1$ and $w_2$ in $T((\R^d))^*\cong T((\R^d)^*)$ (where $T((V))$ is the algebra of formal tensor series on the vector space $V$), then
\begin{equation}\label{eq:shuffle}
    \langle w_1, S(x)\rangle\langle w_2, S(x)\rangle = \langle w_1\shuffle w_2 , S(x)\rangle,
\end{equation}
where $\shuffle:T(\R^d)\times T(\R^d)\to T(\R^d)$ is the shuffle product; the reader is referred to \cite[Section 1.3.3]{cass2024lecture} for more details. 

The above may be extended to a broader class of paths, called geometric $p$-rough path, which are expressed as limits in $p$-variation of a sequence $(\pi_{\lfloor p\rfloor}(S(x^n)))$ of truncated signatures of bounded variation paths $(x^n)$, where $\pi_{\lfloor p\rfloor}$ is the projection map on $T^{\lfloor p\rfloor}(\R^d)$. We denote the space of geometric $p$-rough paths by $G\Omega_p(\R^d)$.

\subsection{One-class support vector machines}
\Glspl{ocsvm} are a widely used class of algorithms for anomaly detection. The central idea is to embed the input space $\X$ into a higher-dimensional space $\h$, with a Hilbert space structure, via a feature map $\varphi:\X\to \h$, and then learn a function $g:\mathcal{X}\to \mathbb{R}$ of the form 
\begin{align}\label{eq:ocsvm_fn}
    x\longmapsto g(x):= \langle w, \varphi(x)\rangle_{\h} -\rho, 
\end{align}
with $(w,\rho)\in \h\times \R$, that takes positive values within a region of the input space and negative values outside, thereby separating the “normal” data from potential anomalies. 

It is well understood that the range of the signature is a subset of a Hilbert space~\citep{salvi2021signature} and the same holds for their truncated versions~\citep{kiraly2019kernels}, making them compatible with this framework. More precisely, let $\{\langle \cdot, \cdot \rangle_{k}:k\in \mathbb N\}$ denote a sequence of inner products on the tensor spaces $(\statespace)^{\otimes k}$, where each $\langle \cdot,\cdot \rangle_k$ is the canonical (Hilbert-Schmidt) inner product derived from a fixed inner product on~$\statespace$, with the convention $(\statespace)^{\otimes 0}=\R$. The subspace  
\begin{equation}\label{eq:SigRKHS}
\h := \left\{  \tensora\in T((\statespace)) : \| \tensora \|_2 < \infty\right\}\subset T((\statespace)),
\end{equation}
with  $\rkhsnorm{\tensora} = \sqrt{\sum_{k=0}^\infty\|a^k\|^2_{k}}$
the $\ell_2$-norm with inner product $\langle A, B\rangle_{\h} =  \sum_{k=0}^{\infty}\langle a^k, b^k\rangle_{k}$, is a Hilbert space. 
For any space $\X$ of $p$-rough paths, the signature kernel \begin{align*}
\kappa(x,y) := \langle S(x), S(y)\rangle_\h\in\R,
\end{align*}
is well defined for all $x,y\in\X$. 
For a truncation level $N\in\mathbb{N}$, we define the truncated inner product $\langle A, B\rangle_{\h_N} := \sum_{k=0}^{N}\langle a^k, b^k\rangle_{k}$ and the associated truncated signature kernel
\begin{equation}\label{eq:truncated_signature_kernel}
\kappa_N(x,y) := \langle S_N(x), S_N(y)\rangle_{\h_N}, 
\end{equation}
which corresponds to the kernel induced by the truncated signature map $S_N$ in~\eqref{eq:truncated_signature}.
The reader may thus think of~$\varphi$ in~\eqref{eq:ocsvm_fn} as being the signature $S$ or its truncated version $S_N$, with $\h$ and $\h_N$ their respective feature spaces.

In the seminal paper by \cite{scholkopf2001estimating}, the hyperplane parametrised by $(w,\rho)$ is obtained on the basis of $n$ observations $\{x_i\}_{i=1}^{n}$ in $\X$ by solving the ``dual" optimisation problem 
\begin{align}\label{eq:ocsvm_optim}
    &\min_{\alpha\in \R^n}\frac{1}{2}\alpha^\top K \alpha, \\
     &\text{subject to } 0\leq \alpha_i \leq \frac{1}{\gamma n}, \quad\text{for all }i=1,\ldots, n\quad e^\top \alpha =1,
\end{align}

where the entries of the matrix $K$ are given by $[K]_{i,j}=\langle \varphi(x_i),\varphi(x_j)\rangle_{\h}$. This yields the acceptance region $\Omega_{\mathrm{ocsvm}} := \{x\in\X: \sum_{i=1}^{n}\alpha_i\kappa(x_i,x)\geq \rho\}$ where $\kappa(x,x')=\langle \varphi(x),\varphi(x')\rangle_{\h}$ and $\rho=\sum_{j=1}^{n}\alpha_j\kappa(x_j,x_i)$ for any $x_i$ such that $0<\alpha_i<\frac{1}{\gamma n}$. 

The effectiveness of \gls{ocsvm} largely depends on the chosen feature map. For novelty detection on path space with signatures, either truncated at some level $N>0$, or untruncated, one can use the kernel tricks developed in~\cite{kiraly2019kernels},~\cite{salvi2021signature} and~\cite{cirone2023neural} for evaluating inner products of signatures. Signature kernels have been applied to hypothesis testing \citep{salvi2021higher, lemercier2021distribution, horvath2023optimal}, to causality~\citep{mantensignature}, to quantitative finance \citep{pannier2024path, cirone2025rough}, and have even emerged as infinite-width limits of neural networks~\citep{cirone2023neural}. They also enable training neural SDEs for time-series generation in fluid dynamics~\citep{salvi2022neural}, computational neuroscience~\citep{holberg2024exact}, and again quantitative finance~\citep{issa2023non, diaz2023neural, hoglund2023neural}.

After solving~\eqref{eq:ocsvm_optim} via sequential minimal optimisation~\citep{chang2011libsvm}, evaluating the \gls{ocsvm} at a new path \(x\) requires computing the score $f(x)=\sum_{i=1}^n \alpha_i \kappa(x_i,x)$. In the case where the feature map is the truncated signature, the primal variables \(w=\sum_{i=1}^n \alpha_i S_N(x_i)\) may be precomputed and the function 
\begin{align*}
    f(x) = \langle w, S_N(x)\rangle_{\h_N}  
\end{align*} 
may then be evaluated in time complexity \(\Oo(d^NL_x)\), where \(d\) is the input path dimension and $L_x$ the number of observations. This amortised evaluation makes \gls{ocsvm} particularly attractive when many candidate paths are to be tested against the same measure~$\mu$.

\subsection{Distance to the expected signature}
If $\mu$ is a measure on path space with a well-defined expected signature $ \ex_{\mu} [S (X)]$, a natural test statistic is obtained by considering the non-negative function $f:\X\to \R_{\geq 0}$ defined by
\begin{align}\label{eq:distance_to_mean}
    f(x) := \rkhsnorm{S_N(x) - \ex_\mu [S_N(X)]},
\end{align}
that is, the distance in $\ell_2$-norm between the signature of \(x\), truncated at level $N$, and the expected truncated signature of the process with law \(\mu\). The function in~\eqref{eq:distance_to_mean} can be rewritten as a linear functional of the signature
\begin{align*}
    f(x)=\langle w, S(x)\rangle, 
    \quad\text{with } w = \sum_{k = 0}^{N} \sum_{I \in \mathcal{I}_k} \left(\langle e_I\shuffle  e_I, \cdot\rangle-
  2 \langle \ex_\mu [S_N(X)], \cdot \rangle + \rkhsnorm{\ex_\mu [S_N(X)]}^2\right),
\end{align*}
where $\mathcal{I}_k=\{1,\ldots,d\}^k$ is the set of multi-indices of length~$k$. 
It may also be rewritten in terms of the truncated signature kernel $\kappa_N$ from~\eqref{eq:truncated_signature_kernel} as 
\begin{align}\label{eq:kernel_trick_distance_to_mean}
    f(x)= \sqrt{\kappa_N(x,x)-2\ex_\mu[\kappa_N(X,x)]+\ex_{\mu}[\kappa_N(X,X')]]},
\end{align}
where \(X'\) is an independent copy of \(X\sim\mu\), and~\eqref{eq:kernel_trick_distance_to_mean} can be evaluated using a kernel trick. After precomputing the terms that do not depend on $x$, each (kernelised and direct) evaluation of the score has the same cost as the ones discussed for \glspl{ocsvm}.
One can generalise this construction by replacing the ambient $\ell_2$-norm with other norms. For example, a ``variance norm'' leads to a Mahalanobis-like distance in feature space. This perspective is developed in more detail in the next section. See also~\cite{cass2024variance} for a related kernelised treatment. 

\subsection{Conformance score}

One of the most widely used nonparametric approaches to anomaly detection is the nearest neighbor (NN) algorithm~\citep{bouman2024unsupervised}. 
Given a dataset of normal reference samples, the NN score of a new observation is defined by its distance to the closest reference point. 
For time series data, this requires specifying a metric on path space. 
A signature-based approach was introduced in~\citet{shao2020dimensionless}, where each path is embedded into a truncated signature, and distances are then computed in this feature space.  

Formally, for a subset \(C \subset \X\), the distance of a point \(x \in \X\) to \(C\) is defined as
\begin{equation*}
    \mathrm{dist}_{\X}(x, C) := \inf_{y \in C} \, \|x - y\|_{\X},
\end{equation*}
where $\|\cdot\|_{\X}$ is a metric on path space. Replacing paths by their signatures, distances are computed with respect to a chosen norm on the feature space $\h$. Given a measure $\mu$ on $\mathcal{X}$, choosing the \emph{variance norm} (Definition~\ref{def:variance_norm}) associated to the measure $\nu=\mu\circ S^{-1}$ on~$\h$, defined by 
\begin{align*}
\|x\|_{\nu\mbox{-}\text{cov}}:=\sup_{\substack{x^*\in \h^*\\ \mathrm{Cov}_\nu(x^*, x^*)\leq 1}} x^*(x),
\end{align*}
yields what~\citet{shao2020dimensionless} refer to as the \emph{conformance score}. In the case of a Gaussian reference distribution $\mu$, the variance norm $\|\cdot\|_{\mu}$ coincides with the norm of the corresponding Cameron-Martin space $\Hh.$ More details can be found in \Cref{ssec:conformance_score} where we develop a general theory of conformance scores. 
Additionally, by the universal approximation property of signatures \cite[Theorem 1.4.7]{cass2024lecture}, the conformance score can itself be reformulated as a linear functional of the signature, thus fitting naturally within the framework introduced above.

\section{Hypothesis testing for novelty detection on path space}

Having introduced signature-based test statistics, we now turn to their use in hypothesis testing. To determine critical values, we derive analytical approximations to their quantiles, exploiting algebraic properties of signatures together with concentration inequalities to obtain non-asymptotic bounds.

\subsection{Sample-free quantile estimates via expected signature}

The baseline approach to quantile estimation is Monte Carlo sampling, where the $\alpha$-quantile is obtained by inverting the empirical distribution function. The resulting estimator, given by the ${\lceil \alpha n\rceil}^{th}$ order statistic converges at a rate $\Oo(n^{-1/2})$~\citep{serfling2009approximation}. In risk management problems, it is well known that the quantile admits a variational characterisation as the minimiser of a \acrlong{cvar} functional. In financial applications, \acrshort{cvar} is often preferred to the value-at-risk (VaR), the quantile itself, due to its convexity and variational properties. Building on this observation, we show that smooth polynomial CVaR surrogates of the pushforward of a measure $\mu$ on path space under the evaluation map $x\mapsto \langle w, S_N(x)\rangle$ admits a deterministic characterisation in terms of the expected signature.

\begin{definition}[CVaR]
For any real-valued random variable
$Z$ with law $\mathbb P$, the conditional value-at-risk $\mathrm{CVaR}_{\alpha}(Z)$ is defined as the average of the $\alpha$-tail
of the probability distribution of $Z$ for a specified confidence level $\alpha \in [0,1]$, that is
\begin{align*}
    \mathrm{CVaR}_\alpha(Z) &:= \ex[Z|Z\geq \mathrm{VaR}_\alpha(Z)],
\end{align*}
where $\mathrm{VaR}_\alpha(Z):=\inf\{r|P_Z(z\leq r)\geq \alpha\}$ is the $\alpha$-quantile of $Z$.
\end{definition}
Its variational formulation~\citep{rockafellar2002conditional} makes it possible to calculate both VaR and CVaR simultaneously:
\begin{lemma}
    For any real-valued random variable $Z$ and any $\alpha \in [0,1)$, we have 
\begin{align*}
\mathrm{CVaR}_\alpha(Z) &=\min_{\eta\in \R} 
    \left\{\eta+ \frac{\ex[[Z-\eta]^+]}{1-\alpha}\right\}.
\end{align*}
\end{lemma}
Let $n\in\N$ and $Q_n:\R\rightarrow\R$ be a polynomial approximation of the max-function $[\cdot]^+$ on any compact interval $[-K,K]\subset \R$ (such an approximation is guaranteed by the Stone-Weierstrass theorem in the compact-open topology) and consider the functions $f_\alpha, f^n_\alpha:\R\to \R$ defined for all $\rho\in\R$ by
$$
f_\alpha(\rho):= \rho + \frac{\ex[[Z-\rho]^+]}{1-\alpha}
\qquad\text{and}\qquad  f^n_\alpha(\rho):= \rho + \frac{\ex[Q_n(Z-\rho)]}{1-\alpha}, 
$$
so that the CVaR and the smoothed CVaR read 
$$
\mathrm{CVaR}_\alpha(Z) = \min_{\rho \in \R} f_\alpha(\rho)
\qquad\text{and}\qquad  \mathrm{CVaR}^n_\alpha(Z) = \min_{\rho \in[-K,K]} f^n_\alpha(\rho).
$$
By the shuffle product property~\eqref{eq:shuffle} of the signature---the multiplication of two linear forms on the signature is still a linear form on the signature---the stochastic optimisation for $\mathrm{CVaR}^n_\alpha(\langle w,S(X)\rangle) $ may be recast as the minimisation of a deterministic function of the expected signature of the process $X$. 

\begin{theorem}\label{thm:cvar_shuffle} Let $Q_n(x)=\sum_{i=0}^{n}a_ix^i$ be a polynomial of degree $n$ and $X$ a stochastic process with law $\mu$ and satisfying $\ex_\mu[S(X)]<\infty$.
For any linear functional $w\in (T^N(\R^d))^*$,
 \begin{align*}
\ex_\mu[Q_n(\langle w, S(X)\rangle-\rho)]=\langle Q_n^{\shuffle}(w-\rho\id), \ex_\mu[S(X)]\rangle,
\end{align*}
where $Q_n^{\shuffle}:(T^N(\R^d))^*\to (T^{nN}(\R^d))^*$ is defined by 
$Q_n^{\shuffle}(\ell) := \sum_{i=0}^{n} a_i \ell ^{\shuffle i}$ and 
\begin{align*}
\mathrm{CVaR}^n_\alpha(\langle w,S(X)\rangle) = \min_{\rho\in[-K,K]} \sum_{m=0}^{n} b_m \rho^m,
  \end{align*}
where the coefficients $b_m$ are given explicitly by
$$
b_m = \delta_{\{m=1\}} + \frac{1}{1 - \alpha} \sum_{i = m}^{n} a_i \binom{i}{m} (-1)^m \left\langle w^{\shuffle (i - m)}, \mathbb{E}_\mu[S(X)] \right\rangle. 
$$
\end{theorem}
\begin{proof}
First, we rewrite the \acrshort{cvar} objective using the shuffle product property~\eqref{eq:shuffle} as 
\begin{align*}
    f^n_\alpha(\rho) &= \rho + \frac{\ex_\mu[Q_n(\langle w,S(X)\rangle-\rho)]}{1-\alpha} \\ 
      &=\rho + \frac{\ex_\mu[Q_n(\langle w-\rho  \id,S(X)\rangle)]}{1-\alpha}\\
      &=\rho + \frac{\ex_\mu[\sum_{i=0}^{n}a_i(\langle w-\rho  \id,S(X)\rangle)^i]}{1-\alpha} \\ 
      &=\rho + \frac{\ex_\mu[\sum_{i=0}^{n}a_i\langle (w-\rho  \id)^{\shuffle i},S(X)\rangle]}{1-\alpha} \\ 
      &=\rho + \frac{\langle \sum_{i=0}^{n}a_i (w-\rho  \id)^{\shuffle i},\ex_\mu[S(X)]\rangle}{1-\alpha}  \\ 
      &= \rho + \frac{\sum_{i=0}^{n}a_i \sum_{k=0}^{i}\binom{i}{k} (-\rho)^{i - k}\langle w^{\shuffle k},\ex_\mu[S(X)]\rangle}{1-\alpha}.
\end{align*}
Then, $f^n_\alpha$ can be rewritten as a polynomial $\sum_{m=0}^{n}b_m \rho^m$ in $\rho$ with coefficients $b_m$ given by
\begin{align*}
    b_m = \delta_{m=1} + \frac{1}{1 - \alpha} \sum_{i = m}^{n} a_i \binom{i}{m} (-1)^m \left\langle w^{\shuffle (i - m)}, \mathbb{E}_\mu[S(X)] \right\rangle.
\end{align*}
\end{proof}
\cite{tsyurmasto2014value} showed that \glspl{ocsvm} can be interpreted as minimising the regularised \acrshort{cvar} (see \Cref{sec:ocsvm_appendix}):
\begin{equation}\label{eq:cvar_opt}
w^*=\arg\min_{w\in \h}\left\{ \mathrm{CVaR}_{\alpha}(-\langle w, \varphi(X)\rangle_{\h})+\frac{1}{2} \|w\|^2_{\h}\right\}.
\end{equation}
If $\varphi$ is chosen as the truncated signature~$S_N$ and \acrshort{cvar} is replaced by a smooth polynomial surrogate, \Cref{thm:cvar_shuffle} allows us to reformulate the stochastic quadratic programming problem in~\eqref{eq:cvar_opt} into a more tractable optimisation task, that amounts to minimising a polynomial
in the variable $\rho$ with coefficients expressed in terms of the expected signature $\ex_\mu[S(X)]$ of the stochastic process $X$.

\subsection{Probabilistic quantile estimates via  transportation-cost inequalities}\label{ssec:probabilistic_estimates}
Throughout this section, we fix $\gamma\in(0, 1)$ and denote by $C^\gamma([0,1];\R^d)$ the Banach space of $\gamma$-H\"older continuous paths $x:[0,1]\rightarrow\R^d$ endowed with the norm  $$
\|x\|_\gamma:=\sup_{t\in[0,1]}|x(t)|+\sup_{t\neq s}\frac{|x(t)-x(s)|}{|t-s|^{\gamma}}.
$$ For technical reasons, we define $C^\gamma([0,1];\R^d)$ as the completion of smooth paths under the H\"older norm which is a Polish space (separable, completely metrisable topological space). We will often write $C^\gamma$ instead of $C^\gamma([0,1];\R^d)$ when there is no risk of confusion.  The space of Borel probability measures on a Polish space $\X$ is denoted by $\mathscr{P}(\X).$
In order to derive error bounds for hypothesis testing, we will assume that reference measures $\mu$, under the null hypothesis, satisfy Transportation-Cost Inequalities (TCIs),
which we now recall, 
and we refer the reader to~\citep{gasteratos2023transportation} for more details on these functional inequalities.

\begin{definition}\label{dfn:Wasserstein} Let $\mathcal{X}$ be a Polish space, 
$c:\mathcal{X}\times\mathcal{X}\rightarrow [0,\infty]$ a measurable function 
and $\mu, \nu\in\mathscr{P}(\mathcal{X})$. 
\begin{enumerate}
\item The transportation cost between~$\mu$ and~$\nu$ with respect to the cost function~$c$ reads
$$
W_{c}(\mu, \nu) :=\inf_{\pi\in\Pi(\mu,\nu)}\iint_{\mathcal{X}\times\mathcal{X}}c(x,y)d\pi(x,y),
$$
where $\Pi(\mu,\nu)$ is the collection of couplings between~$\mu$ and $\nu$:
$$
\Pi(\mu,\nu):=\bigg\{\pi\in\mathscr{P}(\mathcal{X}\times\mathcal{X}): [\pi]_1=\mu,\;[\pi]_2=\nu \bigg\}.
 $$
\item The relative entropy of~$\nu$ with respect to~$\mu$ is given by
\begin{equation*}
	\Hent(\nu\;|\; \mu) :=
 \left\{
 \begin{array}{ll}
	\displaystyle
	\large \int_{\mathcal{X}}\log\left(\frac{d\nu}{d\mu}\right)d\nu, & \text{if }\nu\ll\mu,\\
 +\infty, & \text{otherwise}.
\end{array}
\right.
\end{equation*}
\end{enumerate}
\end{definition}

The following TCIs generalise Talagrand's classical transportation-cost inequality (\cite{talagrand1996transportation}). These functional inequalities imply heavier-than-Gaussian tail bounds and lead to informative error bounds for hypothesis tests that are expressed with respect to relative entropies between the null and alternative hypothesis measures.

\begin{definition}\label{def:TCI} 
Let $\mathcal{X}$ be a Polish space,   $\mu\in\mathscr{P}(\mathcal{X})$,  $c:\mathcal{X}\times\mathcal{X}\rightarrow[0,\infty]$ a measurable function with $c(x,x)=0$ for all~$x\in\mathcal{X}$ and $a:[0,\infty]\rightarrow[0,\infty]$ a lower semicontinuous function with $a(0)=0$.
We say that~$\mu$ satisfies the $(a, c)$-TCI (and write $\mu\in\mathscr{T}_{a}(c)$) with cost function~$c$ and deviation function~$a$ if, 
for all $\mathscr{P}(\mathcal{X})\ni\nu\ll\mu$,
\begin{equation}\label{alphap}
a\bigg(W_c(\mu, \nu)\bigg)\leq \Hent(\nu\;|\; \mu).
\end{equation}
\end{definition}

For the rest of this section we make the following standing assumption: 

\begin{assumption}\label{assumption:mu} For $\gamma\in(0,1)$ we take $\X=C^\gamma([0,1];\R^d).$ The reference probability measure $\mu\in\mathscr{P}(C^\gamma([0,1];\R^d))$ satisfies the $(a, c)$-TCI with cost function $$
c(x,y):=\|x-y\|^p_\gamma,
\qquad\text{for all }
x,y\in C^\gamma
$$
and a continuous, strictly increasing \textit{deviation function} $a:[0,\infty]\rightarrow[0,\infty]$ with $a(0)=0$.
    \end{assumption}

 For a sample path $X$ we consider the hypothesis test: 
\begin{equation}\label{eq:HypothesisTestTCI}
    \begin{aligned}
        &H_0: X\sim\mu,\\
        &
        H_1: X\sim\nu\neq\mu,
    \end{aligned}
\end{equation}
where $\nu\in \mathscr{P}(C^\gamma([0,1];\R^d))$ is an arbitrary probability measure on $\X.$ The test statistic is chosen to be the decision function $F:C^\gamma([0,1];\R^d)\rightarrow \{0,1\}$
defined by 
 $$
F(x):=\mathds{1}_{\{\langle w, S_N(x)\rangle> r\}}(x),
\qquad\text{for } x\in C^\gamma\left([0,1];\R^d\right), N\in\N,
$$
where $w\in\h$ lies in the signature RKHS and solves the corresponding one-class SVM optimisation problem and $r\in\R$ is arbitrary but will later be chosen to be an $\alpha$-quantile as in~\eqref{eq:region_cvar}. This class of test statistics includes all of the examples described in Section~\ref{sec:test_statistcs}.
The type-$\mathrm{I}$ and type-$\mathrm{II}$ errors are thus given by
\begin{equation}\label{eq:typeI}
P_{\mathrm{I}}:=\pr\bigg( F(X)=1| X\sim\mu   \bigg)=\mu\bigg( \langle w, S_N(x)\rangle> r    \bigg),
\end{equation}

\begin{equation}\label{eq:typeII}
P_{\mathrm{II}}:=\pr\bigg( F(X)=0| X\sim\nu   \bigg)=\nu\bigg( \langle w, S_N(x)\rangle\leq  r    \bigg).
\end{equation}

For the rest of this section we assume that $r>0$ and derive upper and lower bounds for the type-$\mathrm{I}$ and type-$\mathrm{II}$ errors in Theorems~\ref{thm:Type-II_error}, \ref{thm:TypeIerror} respectively. 
The case $r<0$ will be discussed in Remark~\ref{rem:rnegative} below.
Given $w, r$, Theorem~\ref{thm:TypeIerror} provides the upper bound for the type-$\mathrm{I}$ error:
$$
P_{\mathrm{I}}\leq C_2
\bigg[\exp\bigg\{-\frac{C^2_1}{2}\bigg(\frac{r}{\sqrt{N}d^{N}C^{N/2}\|w\|_{\h}}\bigg)^{2p}\vee \bigg(\frac{r}{\sqrt{N}d^{N}C^{N/2}\|w\|_{\h}}\bigg)^{\frac{2p}{N}}      \bigg\}\bigg], 
$$
where $d$ is the dimension of the paths, $C_1, C_2$ are absolute constants that depend on the exponential moments of $\mu$ and the deviation function $a$ is assumed to be super-quadratic near the origin (an assumption that is satisfied in all the examples of interest).

\begin{example}[RDEs with Gaussian drivers]\label{example:RDEs}
An important class of measures that satisfy $(a, c)$-TCIs is given by laws of solutions to Rough Differential Equations (RDEs) driven by Gaussian rough paths. In particular, let $\beta\in(\tfrac{1}{3}, 1)$ and $X$ be a Gaussian process with almost surely $\beta$-H\"older continuous paths that has a natural geometric $\beta$-rough path lift~$\mathbf{X}.$
If the corresponding Cameron-Martin space satisfies 
$\h\hookrightarrow C^{q\textnormal{-var}}([0,T];\R^d)$
for some~$q>0$, with $\frac{1}{q}+\beta>1$, then the law $\mu\in\mathscr{P}(C^{1/\beta\textnormal{-var}}([0,T];\R^d))$ of the solution $Y$ of the RDE 
$$
dY_t=V(Y_t)d\mathbf{X}_t,
$$
where $V$ is a Lipschitz vector field with constant $\gamma>1/\beta$, satisfies a $\mathcal{T}_a(c)$-inequality~\cite[Theorem~3.3]{gasteratos2023transportation} with
$$
c(\mathbf{x},\mathbf{y})=\|\mathbf{x}-\mathbf{y}\|_{\frac{1}{\beta}\text{-var}}^{\frac{1}{q}}
\qquad\text{and}\qquad
a(t)=t^2\wedge t^{2q}.
$$
In view of~\cite[Corollary 3.5]{gasteratos2023transportation}, the latter implies the tail estimates in~\citep{cass2013integrability}. While the results in this section do not depend on the particular choice of path topologies, the aforementioned TCI for solution laws of RDEs is valid on finite $p$-variation topologies as it leverages complementary Young regularity. Thus, all our results can be applied to this measure mutatis mutandis i.e. by substituting the H\"older space and norm in Assumption~\ref{assumption:mu} and proofs by the corresponding $1/\beta$-variation topology. 
\end{example}

\noindent From the fact that $\mu\in\mathscr{T}_a(c)$ we will establish lower bounds for the type-$\mathrm{II}$ error. 
To do so, we rely on the following dual characterisation of the transportation cost~$W_c$.
\begin{lemma}\label{lem:Wcdual} Let $p\in(0,1]$, $\gamma\in(0,1)$ as in Assumption~\ref{assumption:mu} and 
$$
\Cc^{p,\gamma}:=\bigg\{ f: C^\gamma([0,1];\R^d)\rightarrow\R: \inf_{x\neq y\in C^\gamma}\frac{|f(x)-f(y)|}{\|x-y\|^p_\gamma}\leq 1   \bigg\}.$$
Then, for any $\mu, \nu\in \mathscr{P}(C^\gamma([0,1];\R^d))$ with finite first moments,
\begin{equation*}
W_c(\mu, \nu) = \sup_{f\in \Cc^{p,\gamma} }\bigg\{\ex_{\mu}[f]-\ex_{\nu}[f]\bigg\}.
\end{equation*}
\end{lemma}
\begin{proof} By Kantorovich duality~\cite[Theorem 2.2]{gozlan2010transport},
$$W_c(\mu,\nu)=\sup\bigg\{   \ex_\mu[f]+\ex_\nu[g] \bigg| f\in L^1(\mu), g\in L^1(\nu), f(x)+g(y)\leq c(x,y)=\|x-y\|^p_{\gamma}\bigg\}.        $$
For any  fixed $g\in L^1(\mu)$, 
using triangle and Young's product inequalities, the function 
$$
f^*(y) := \inf_{x\in C^\gamma}\bigg\{ \|x-y\|^p_{\gamma}-g(x) \bigg\}$$
satisfies
\begin{enumerate}
\item With $(1/p)^*:=(1-p)^{-1}$ and for all $y\in C^\gamma([0,1];\R^d)$ we have $f^*(y)\leq \|y\|^p_{\gamma}-g(0)\leq p\|y\|_\gamma+\frac{1}{(1/p)^*}+|g(0)|$, so that $f^*\in L^1(\mu)$ since~$\mu$ has a finite first moment;
\item For any $x_2\neq x_1\in C^\gamma$, $|f^*(x_2)-f^*(x_1)|\leq \|x_2-x_1\|^p_\gamma$, so that $f^*\in \Cc^{p,\gamma}$. 
Indeed, for any $\delta>0$ there exists $\bar{x}\in C^\gamma$ such that 
$f^*(x_1)>\|\bar{x}-x_1\|_{\gamma}^p-g(\bar{x})-\delta$.
Thus, 
$$
f(x_2)-f(x_1)\leq 
\|\bar{x}-x_2\|_{\gamma}^p - g(\bar{x}) - \|\bar{x}-x_1\|_{\gamma}^p + g(\bar{x})+\delta\leq \|x_1-x_2\|_{\gamma}^p+\delta,
$$
since the inequality $x^p-y^p\leq |x-y|^p$ holds for real numbers $x,y\geq 0$; we also used the triangle inequality for the norm $\|\cdot\|_\gamma$ and the monotonicity of the function $x\mapsto x^p.$ Interchanging the roles of $x_1, x_2$ we obtain the reverse inequality 
$$ f(x_2)-f(x_1)>-\|x_2-x_1\|^p_\gamma;$$
\item Out of all functions $f\in\mathscr{C}^{p,\gamma}$ that satisfy $f(x)+g(y)\leq c(x,y)$ we have
$\ex_\mu[f]\leq \ex_\mu[f^*]$.
\end{enumerate}
From these facts we deduce that 
\begin{equation}\label{eq:Wassersteindual}
\begin{aligned}
W_c(\mu,\nu)
&=\sup_{g\in\mathscr{C}^{p,\gamma}}\sup_{\substack{f\in\mathscr{C}^{p,\gamma}\\ f\oplus g\leq c}} \bigg\{ \ex_\mu[f]+\ex_\nu[g]  \bigg\}\\
& = \sup_{g, f\in \Cc^{p,\gamma}}\sup_{f\oplus g\leq c} \bigg\{ \ex_\mu[f]+\ex_\nu[g]  \bigg\}\\
&=\sup_{f\in\mathscr{C}^{p,\gamma}}\sup_{\substack{g\in\mathscr{C}^{p,\gamma}\\ f\oplus g\leq c}}  \bigg\{ \ex_\mu[f]+\ex_\nu[g]  \bigg\}.
\end{aligned}
\end{equation}
Now, for any fixed $f\in \Cc^{p,\gamma}$, we can optimise the last equality by taking $$
g(y)=g^*(y):=\inf_{x\in C^\gamma}\{ \|x-y\|^p_\gamma-f(x)      \}
$$ as before. 
On the one hand, 
$g^*(y)\leq -f(y)$ by definition of $g^*$;
on the other hand, since $f\in \Cc^{p,\gamma}$ we must have, for all $x,y$, 
$\|x-y\|^p_\gamma-f(x)+f(y)\geq 0$. This implies
$g^*(y)=\inf_{x\in C^\gamma}\{ \|x-y\|^p_\gamma-f(x)      \}\geq -f(y)$
namely $g^*=-f$.
Returning to~\eqref{eq:Wassersteindual} we obtain
$$
W_c(\mu,\nu)=  \sup_{f\in \Cc^{p,\gamma}, g=g^*} \bigg\{ \ex_\mu[f]+\ex_\nu[g]  \bigg\} = \sup_{f\in \Cc^{p,\gamma}} \bigg\{ \ex_\mu[f]-\ex_\nu[f]  \bigg\},$$
which completes the proof.
\end{proof}

\noindent With this characterisation of $W_c$ we derive a lower bound for the type-$\mathrm{II}$ error using the following:
\begin{corollary}\label{cor:momentBound} 
Let $\mu$ satisfy Assumption~\ref{Assumption:iid}. Then for any other measure $\nu\in\mathscr{P}(C^\gamma)$, with finite first moments and any $f\in \Cc^{p,\gamma}$,
$$
\ex_\nu[f]\leq a^{-1}\bigg(  \Hent(\nu|\mu)   \bigg)+\ex_\mu[f].
$$
\end{corollary}
\begin{proof} The bound follows directly by the symmetry of $W_c$, Lemma~\ref{lem:Wcdual} and the fact that $\mu\in\mathscr{T}_a(c)$ if and only if $a(W_c(\nu, \mu))\leq \Hent(\nu|\mu)$ per Definition~\ref{def:TCI}
\end{proof}

We are now ready to state novel type-$\mathrm{II}$ errors for the hypothesis test~\eqref{eq:HypothesisTestTCI} in terms of the relative entropy.

\begin{theorem}[Type-II error] 
\label{thm:Type-II_error}
With $p,\gamma, c, \mu$ as in Assumption~\ref{assumption:mu} and $w, N\in\N, r>0$ as in~\eqref{eq:typeII} and for any $\nu\in\mathscr{P}(C^\gamma)$ with finite first moment, the following bounds hold:
\begin{equation*}   P_{\mathrm{II}}\geq 1-\bigg(\frac{\|w\|_{\h}}{r}\bigg)^{\frac{p}{N}}Cd^pN^{\frac{1}{2Np}}\bigg\{ 1-\frac{1}{N}+ \frac{1}{N}\bigg(a^{-1}\big(\Hent(\nu|\mu)\big)  +\ex_\mu[\|X\|^{p}_\gamma]\bigg)  \bigg\}\;,
\end{equation*}
if $a^{-1}\big(\Hent(\nu|\mu)\big)  +\ex_\mu[\|X\|^{p}_\gamma]\leq 1$,
and
\begin{equation*}
   P_{\mathrm{II}}\geq
  1-\bigg(\frac{\|w\|_{\h}}{r}\bigg)^{\frac{p}{N}}Cd^pN^{\frac{1}{2Np}}\bigg\{ a^{-1}\big(\Hent(\nu|\mu)\big)  +\ex_\mu[\|X\|^{p}_\gamma]  \bigg\}\;,\; \textnormal{if}\;a^{-1}\big(\Hent(\nu|\mu)\big)  +\ex_\mu[\|X\|^{p}_\gamma]> 1, 
\end{equation*}
where $d$ is the dimension of the path~$X$ and $C>0$ a constant independent of $d$.
\end{theorem}
\begin{proof}
By Chebyshev's inequality
\begin{equation}\label{eq:PIIest1}
P_{\mathrm{II}}=1-\nu\bigg( \langle w, S_N(x)\rangle>  r     \bigg)\geq 1-\bigg(\frac{\|w\|_{\h}}{r}\bigg)^{\frac{p}{N}}\ex_\nu\left[|S_N(X)|^{\frac{p}{N}}\right].
   \end{equation} 
In view of
\cite[Section 1]{boedihardjo2018decay} and references therein, we can obtain a pathwise estimate 
for the signature, truncated at level $N$, of any $d$-dimensional, $\gamma$-H\"older path $X$ with $\gamma\in(0, \frac{1}{2}]$. The latter takes the form:
\begin{align*}
\big|S_N(X)\big|^2
 &\leq \sum_{k=1}^{N}\bigg(\sum_{|w|=k} S^w_{0,1}(X)\bigg)^2
 \\& \leq \sum_{k=1}^{N} C^2(k)\|X\|^{2k}_\gamma\bigg(\sum_{|w|=k}\frac{1}{\Gamma(k\gamma+1)}\bigg)^2\leq \|X\|^{2}_\gamma\vee\|X\|^{2N}_\gamma\sum_{k=1}^{N} C^2(k)d^{2k},
\end{align*}
where~$|w|$ is the length of a word $w$, $\Gamma$ denotes the Gamma function and $C(k)=\Oo((1+c)^k)$ is an increasing polynomially growing function in $k$.
Hence the last estimate yields 
\begin{equation}\label{eq:SigPathwiseBound}
    \begin{aligned}
\big|S_N(X)\big|^2
&\leq \|X\|^{2}_\gamma \vee\|X\|^{2N}_\gamma Nd^{2N}(1+\rho)^N,
    \end{aligned}
\end{equation}
for some constant~$\rho>0.$
Taking expectations it follows that
\begin{align*}
\ex_\nu\bigg[ \big|S_N(X)\big|^{\frac{p}{N}}\bigg]
 & \leq d^{p}(1+\rho)^{p}N^{\frac{p}{2N}}\ex_\nu\left[\|X\|^{\frac{p}{N}}_\gamma\vee\|X\|^p_\gamma \right]\\
 & \leq C d^{p}N^{\frac{p}{2N}}\ex_\nu\left[\|X\|^{\frac{p}{N}}_\gamma\right]\vee
\ex_\nu\left[\|X\|^p_\gamma \right],
\end{align*} 
for some unimportant constant $C>0$.
By Young's product inequality, we have
$$
\ex_\nu\left[\|X\|^{\frac{p}{N}}_\gamma\right]
\leq \frac{1}{\frac{N}{N-1}}+\frac{1}{N}\ex[\|X\|^{p}_\gamma]=1-\frac{1}{N}+ \frac{1}{N}\ex_\nu[\|X\|^{p}_\gamma].$$
Now, notice that the function $f(x)=\|x\|_\gamma^{p}$ is in $\Cc^{p,\gamma}$. In view of Corollary~\ref{cor:momentBound} then
$$
\ex_\nu\left[\|X\|^{\frac{p}{N}}_\gamma\right]
\leq 1-\frac{1}{N}+ \frac{1}{N}\bigg(a^{-1}\big(\Hent(\nu|\mu)\big)  +\ex_\mu[\|X\|^{p}_\gamma]\bigg).$$
Combining all these estimates we see that $\ex_\nu\bigg[ \big|S_N(X)\big|^{\frac{p}{N}}\bigg]$ is bounded from above by
\begin{equation*}
\begin{aligned}
    Cd^pN^{\frac{p}{2N}}\max\bigg\{ 1-\frac{1}{N}+ \frac{1}{N}\bigg(a^{-1}\big(\Hent(\nu|\mu)\big)  +\ex_\mu[\|X\|^{p}_\gamma]\bigg)\;,\; \bigg(a^{-1}\big(\Hent(\nu|\mu)\big)  +\ex_\mu[\|X\|^{p}_\gamma]\bigg)    \bigg\}.
\end{aligned}    
\end{equation*}
The proof then follows by combining this with~\eqref{eq:PIIest1}.
\end{proof}

\begin{remark} A few observations on the lower bound for the type-$\mathrm{II}$ error in Theorem~\ref{thm:Type-II_error}:
\begin{enumerate}
    \item If the measure $\nu$ from the alternative hypothesis is not absolutely continuous with respect to $\mu$ then $\Hent(\nu|\mu)=\infty$ and, since $a^{-1}$ is (typically) increasing to $\infty$, the lower bound is then uninformative. Nevertheless, we believe that the regime in which $\Hent(\nu|\mu)<\infty$ is more interesting for application purposes. Indeed, if $\nu, \mu$ are absolutely continuous laws on path space then, a priori, it is more challenging to tell them apart (or detect novel observations in terms of these two alternatives). It is precisely this setting in which the type-$\mathrm{II}$ lower bound becomes informative. Nevertheless, our numerical experiments work well even in settings where the two alternatives are mutually singular; see Section~\ref{sec:anomalousDiffs} for more details.
    \item In estimating the type-$\mathrm{II}$ error, we made a choice to raise the signature to $p/N$ (below~\eqref{eq:SigPathwiseBound}). This choice was made to avoid exponential growth in the dimension~$d$ which appears on the right-hand side of~\eqref{eq:SigPathwiseBound}.
    \end{enumerate}
\end{remark}

We turn to the analysis of type-$\mathrm{I}$ errors~\eqref{eq:typeI}. 
Assuming for now that~$\mu$ is a Gaussian measure, it is possible to obtain estimates on the probability of the rejection region~\eqref{rejection_region} under $H_0$. In particular, for $\X=C^{p\text{-var}}([0,T];\R^d)$,
$\mu\in\mathscr{P}(\X)$ a Gaussian probability measure on path space and for some $N\in\N$ consider the truncated signature feature map $\varphi=S_N$ of level~$N$.  
The probability that a new sample point from~$\mu$ falls in the rejection region satisfies
$$
\mu\left(\Omega_r\right)
=\mu\bigg[  \langle w,S_N(X)\rangle> r    \bigg]\leq \exp\bigg\{-C\bigg(\frac{r}{\|w\|_{\h}}\bigg)^{\frac{2}{N}}\bigg\}.
$$
In turn, for a given significance level $\alpha\in(0,1)$, the following estimate on~$r$ holds:
$$
r\geq  C^{-1}\|w\|_{\h}\log(1/\alpha)^{\frac{N}{2}}.
$$
Such estimates can be performed not just for Gaussian measures~$\mu$ but also for laws of Gaussian functionals such as solutions of RDEs or in general measures that satisfy transportation-cost inequalities. A broad list of examples of such measures can be found in~\citep{gasteratos2023transportation}; see also Example~\ref{example:RDEs} above. 

Our next result provides type-$\mathrm{I}$ error upper bounds and consequently estimates on the threshold $r$ of the rejection region $\Omega_r$ \eqref{rejection_region} at fixed significance levels.

\begin{theorem}[Type-I error]\label{thm:TypeIerror} Let $\alpha\in(0,1), N\in\N,\gamma\in(0,1), p\in(0,1]$ and $S_N: C^\gamma([0,1];\R)\rightarrow T^{(N)}(\R)$ the truncated signature of level $N$. Let $\mu\in\mathscr{P}(C^\gamma([0,1];\R))$ be a probability measure satisfying Assumption~\ref{assumption:mu}. 
If there exist $t_0, C>0$ such that for all $t<t_0$  $a(t)\geq C t^2$,
then the following holds:

For all $w\in T^{(N)}(\R)$ and with probability $\alpha \in (0,1)$ over the draw of a random sample from $\mu$, a sample path $X$ that satisfies 
\begin{equation}\label{eq:rstarddim}
\langle w, S_N(X)\rangle> r^*:=\|w\|_{\h}\sqrt{N}d^{N}C(N)\max\bigg\{ \bigg[\frac{2}{C_1}\log\bigg(\frac{C_2}{\alpha}   \bigg)\bigg]^{\frac{1}{2p}},  \bigg[\frac{2}{C_1}\log\bigg(\frac{C_2}{\alpha}   \bigg)\bigg]^{\frac{N}{2p}} \bigg\}
\end{equation}
falls in the rejection region $\Omega_{r^*}$ \eqref{rejection_region}. The constants $C_1, C_2$ do not depend on  $N$ and satisfy
$$
C_1\leq \bigg( 2(\ex_\mu\|X\|^p_{\gamma}+C)\bigg)^{-\frac{1}{2}}
\qquad\text{and}\qquad
C_2=\ex_\mu\bigg[\exp\bigg(\frac{C_1^2}{2}\|X\|^{2p}_{\gamma}\|\bigg)\bigg].
$$
\end{theorem}

\begin{proof} 
The pathwise estimate~\eqref{eq:SigPathwiseBound} furnishes
 \begin{align*}
\alpha = \mu(\Omega_r)
&\leq\mu\bigg( |S_N(X)|\|w\|_{\h}> r   \bigg)
  \leq \mu\bigg( \|X\|_\gamma \vee\|X\|^{N}_\gamma >\bigg(\frac{r}{\sqrt{N}d^{N}(1+\rho)^{N/2}\|w\|_{\h}}\bigg)\\
 & \leq \mu\bigg( \|X\|_\gamma>\bigg(\frac{r}{\sqrt{N}d^{N}C^{N/2}\|w\|_{\h}}\bigg)\vee \bigg(\frac{r}{\sqrt{N}d^{N}C^{N/2}\|w\|_{\h}}\bigg)^{\frac{1}{N}}      \bigg).
\end{align*}
From the assumptions on $\mu,a$,~\cite[Corollary 2.6(ii)]{gasteratos2023transportation} and Chebyshev's inequality, the latter is bounded from above by

$$ C_2
\bigg[\exp\bigg\{-\frac{C^2_1}{2}\bigg(\frac{r}{\sqrt{N}d^{N}C^{N/2}\|w\|_{\h}}\bigg)^{2p}\vee \bigg(\frac{r}{\sqrt{N}d^{N}C^{N/2}\|w\|_{\h}}\bigg)^{\frac{2p}{N}}      \bigg\}\bigg], 
$$
from which we deduce that 
$$
\log(\alpha/C_2)\leq-\frac{C^2_1}{2}\bigg(\frac{r}{\sqrt{N}d^{N}C^{N/2}\|w\|_{\h}}\bigg)^{2p}\vee \bigg(\frac{r}{\sqrt{N}d^{N}C^{N/2}\|w\|_{\h}}\bigg)^{\frac{2p}{N}} .  $$
Solving this inequality for $r$, we are led to the desired conclusion.
\end{proof}

\begin{remark}[Tradeoffs related to $N$]
From the estimate~\eqref{eq:rstarddim} we observe that, at fixed significance level $\alpha$, the rejection threshold $r^*=r^*(N)$ grows as the signature truncation level $N$ grows. In particular, rejecting~$H_0$ becomes harder since larger observations become more typical (indeed, tails of the signature become "fatter" as~$N$ increases). On the other hand, larger values of~$N$ lead to higher robustness for test statistics. Indeed, since linear maps of the signature are universal, larger values of~$N$ lead to larger classes of admissible test statistics on path space.
\end{remark} 

\begin{remark}[Error bounds when $r<0$]\label{rem:rnegative}
We assumed above $r>0$ and derived upper/lower bounds for type-$\mathrm{I}/\mathrm{II}$ errors respectively. 
In the case $r<0$, it is easy to check that our techniques provide symmetric  lower/upper bounds for type-$\mathrm{I}/\mathrm{II}$ errors respectively.
\end{remark}

\section{Experiments}

In this section, we evaluate the type-I error and the statistical power of signature-based test statistics, using synthetic anomalous diffusion data and real-world molecular biology data.

\subsection{Anomalous diffusions}\label{sec:anomalousDiffs}

We consider the problem of distinguishing Brownian motion (BM) from BM perturbed by an impulsive spike occurring at a uniformly random time. This example, connected to the theory of Brownian fast points \citep{davis1985brownian}, has been introduced by \citet{davis1984randomly} and \citet{hitchcock1992some}, and has also been empirically examined by \citet{cass2024variance}. For a given sample path $X$, we consider the hypothesis test
\begin{equation*}
H_0: X\sim\mu  \text{ against } 
        H_1: X\sim\nu_\varepsilon,
\end{equation*}
 with~$\mu$ the law of Brownian motion on $[0,2]$ and $\nu_\varepsilon$ that of a process with a spike. Following the presentation in \cite[Chapter 6]{hitchcock1992some},
 define the alternative process
$$
X^\varepsilon_t=B_t+\varepsilon\sqrt{[t-\theta]^+}\wedge 1, \quad\text{for } t\in [0,2],
$$
where $\theta$ is a random time drawn from the uniform distribution on $[0,1]$, $\varepsilon>0$ controls the magnitude of the spike, and $[x]^+=\max\{x,0\}$. Our numerical studies show how our detection procedure—based on signature transforms of the paths—can distinguish between~$\mu$ and~$\nu_\varepsilon$ for different intensities~$\varepsilon$ using the Area Under the Receiver Operating Characteristic curve (AUROC) as performance metric. We expect that as $\varepsilon$ increases, the spike gets more pronounced and produces a clearer non-Brownian signature, leading to higher detection power. When $\varepsilon \approx 0$, the spike is subtle so detection is more challenging.

\citet{hitchcock1992some} established that there is a critical value for the intensity parameter $\varepsilon$. When $\varepsilon < \sqrt{8}$, the measure corresponding to the process with a spike remains absolutely continuous with respect to the measure of standard Brownian motion. Conversely, when $\varepsilon > \sqrt{8}$ the two measures become singular. In our experiments, using the distance to the expected signature as test statistic, we do not observe an abrupt phase transition precisely at $\varepsilon=\sqrt{8}$. Instead, the AUROC improves gradually as~$\varepsilon$ increases 
(Figure~\ref{fig:auroc}). 

\begin{figure}[h]
    \centering
\begin{subfigure}{0.45\textwidth}
    \centering \includegraphics[width=1.\linewidth]{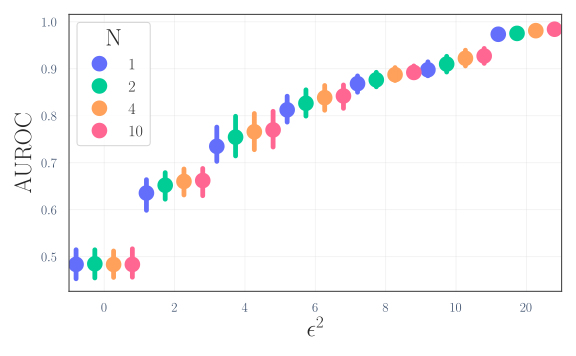}
\caption{}\label{fig:auroc}
\end{subfigure}
\begin{subfigure}{0.45\textwidth}
    \centering
\includegraphics[width=1.\linewidth]{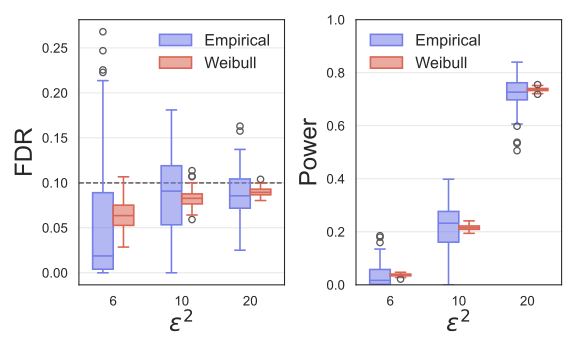}
\caption{}\label{fig:fdr}
\end{subfigure}
\begin{subfigure}{0.45\textwidth}
    \centering
\includegraphics[width=1.\linewidth]{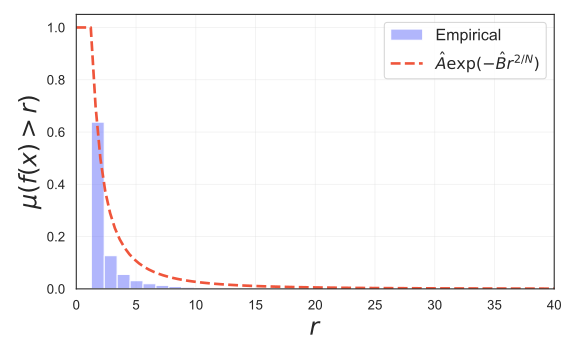}
\caption{}\label{fig:weibull_fit}
\end{subfigure}
\begin{subfigure}{0.45\textwidth}
    \centering
\includegraphics[width=1.\linewidth]{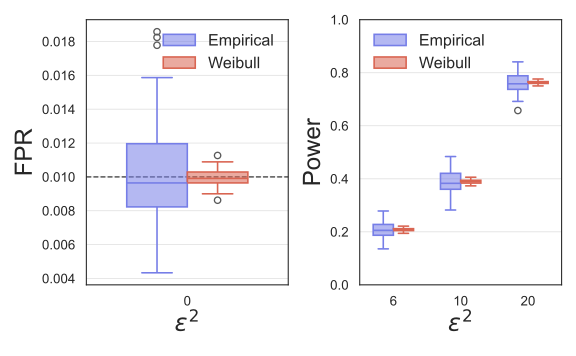}
\caption{}\label{fig:fpr}
\end{subfigure}
    \caption{\textbf{Brownian motion perturbed by a  spike.} (a) AUROC as a function of the spike intensity for the distance to the expected signature. (b) Multiple testing false discovery rate and power at level $\alpha=0.1$ comparing empirical p-values (from 1,000 samples) with p-values obtained from a Weibull tail-bound fitted using 100,000 samples. (c) Weibull tail-bound fit (d) Single hypothesis false positive rate and power at level $\alpha=0.01$. }
    \label{fig:placeholder}
\end{figure}

Next, to evaluate the type-I error rate and the power, we simulate $J=100$ researchers. Each researcher $j$ has an independent fixed dataset $D_j$ of $n=2000$ Brownian motion sample paths and $L=50$ test sets $D^\text{test}_{j,l}$ 
(for $l=1,\ldots, L$)
of $n_\text{test}=1000$ Brownian motion sample paths, with $10\%$ of outliers (Brownian motion perturbed by a spike). Each researcher divides its reference dataset $D_j$ into $1000$ sample paths to estimate the expected signature of Brownian motion and $1000$ sample paths to estimate the p-values of all observations across all test datasets $D_{j,l}$. To measure the false discovery rate (FDR), the Benjamini-Hochberg procedure~\citep{benjamini1995controlling} is applied to all p-values in $D_{j,l}$ at level $\alpha=0.10$, the false positive rate (FPR) is then estimated, and averaged over the $L$ test datasets. Similarly, a statistical estimate for the power is obtained. To estimate the false positive rate, the raw  p-values are thresholded at level $\alpha=0.01$. 

As can be seen on Figures~\ref{fig:fdr} and~\ref{fig:fpr}, which show the FDR and FPR for different values of the signal strength $\varepsilon^2$, both methods control the marginal FDR and FPR at the desired levels. We also compared to a parametric approach (Weibull), where we fit a curve $A\exp(-Br^{2/N})$ to the tail of the reference measure estimated with a relatively large number of samples, leveraging the results from \Cref{ssec:probabilistic_estimates}. As can be seen on  on Figure~\ref{fig:fdr}, this parametric approach controls the conditional FDR and FPR for a greater proportion of researchers. As expected, the false positive rate exhibits a lower variance (Figure~\ref{fig:fpr}) when the Weibull parametric approach is used.

\begin{figure}[h]
    \centering
    \begin{subfigure}{0.49\textwidth}
     \centering
         \includegraphics[width=1\linewidth]{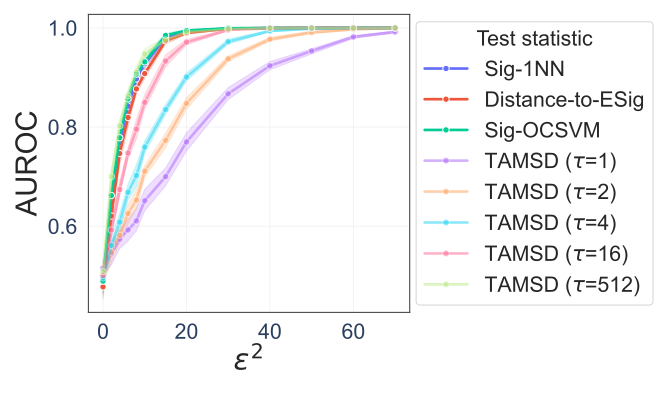}
    \caption{}
    \label{fig:placeholder}
    \end{subfigure}
       \begin{subfigure}{0.49\textwidth}
     \centering
         \includegraphics[width=1\linewidth]{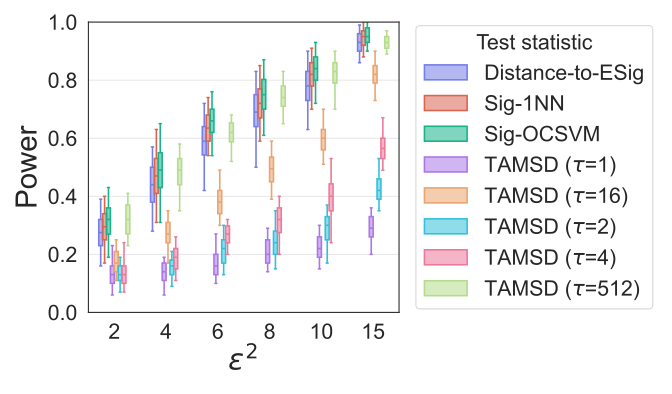}
    \caption{}
    \label{fig:placeholder}
    \end{subfigure}
    \caption{\textbf{Comparison of different test statistics.} Comparison of \acrshort{ocsvm}, conformance score, distance to the expected signature, and TAMSD with $\tau\in\{1,2,4,16,512\}$.}
\end{figure}

Finally, we compare our signature-based test statistics with another test statistic for detecting anomalous diffusions from the literature \citep{sikora2017mean} given by the time-averaged mean square displacement (TAMSD)
\begin{align}
    M_N(\tau) := \frac{1}{N-\tau}\sum_{j=1}^{N-\tau}(X(j+\tau)-X(j))^2.
\end{align}
The TAMSD scales as $M_N(\tau)\sim \tau ^\alpha$ with $\alpha=2H$ for fractional Brownian motion with Hurst index $H \in (0,1)$.
Brownian motion has a linear power-law growth of the mean squared displacement in the course of time.

\subsection{Anomaly detection in RNA direct sequencing data}

In molecular biology, a central challenge is to detect non-canonical nucleotides on RNA molecules, i.e. bases chemically modified relative to the four standard nucleotides (A, U, G, C). Mapping such modifications has become a major research focus, as they have been found to modulate crucial cellular processes and are increasingly linked to human diseases, ranging from neurodevelopmental disorders to cancer. 
Recently, significant progress has been achieved by leveraging Nanopore direct sequencing data \citep{white2022modification,furlan2021computational}. Nanopore sequencing transduces long RNA polymers into an electrical signal and chemical modifications have been shown to induce characteristic alterations in the measured signal. Most prior approaches have focused on training neural network based classifiers on known modification types. However, since many modifications remain unknown or poorly characterised, this is also naturally cast as a novelty detection task \cite{lemercier2025path}.

\begin{figure}[h]
    \centering
    \includegraphics[width=1\linewidth]{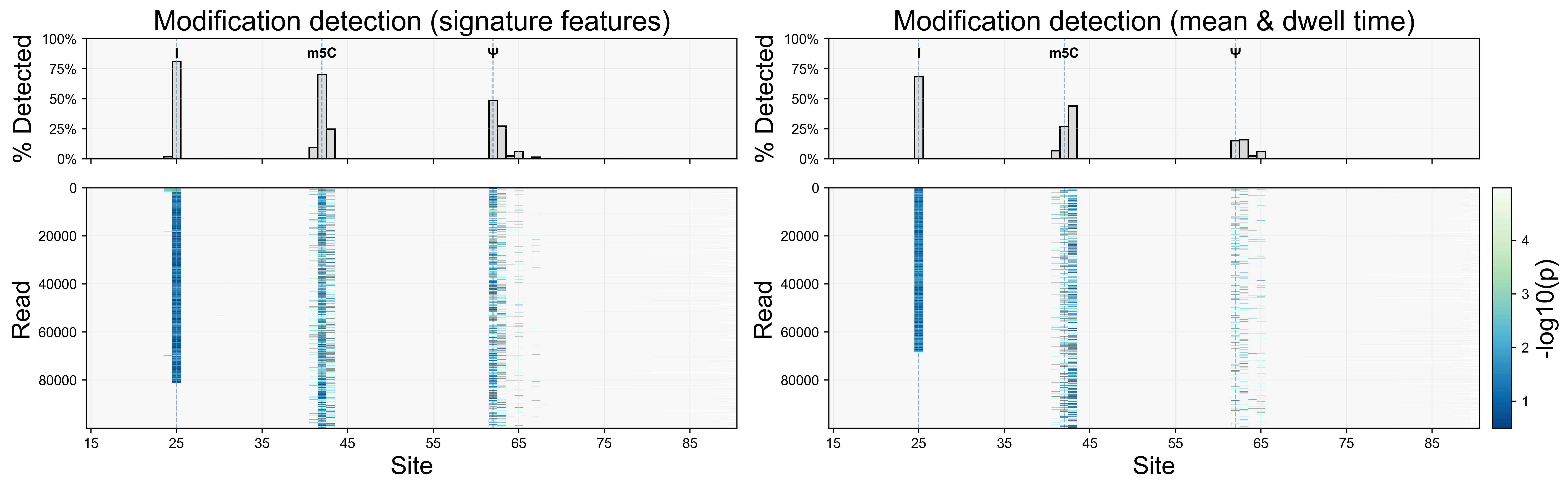}
    \caption{\textbf{Modification detection in nanopore reads with \acrshort{ocsvm}.} Bottom: per-read p-values (with multiple testing correction) at each site; non-significant p-values at level 0.20 are rendered light grey. Top: for each site, the proportion of significant reads. \emph{Left:} signature features. \textit{Right:} mean-current and dwell-time features.}
    \label{fig:RNA}
\end{figure}

Here, we analyse short synthetic oligonucleotides, i.e., RNA molecules chemically synthesised in a laboratory. In this controlled setting, three distinct chemical modifications were introduced at defined positions within a 100-nucleotide sequence \citep{leger2021rna}
\begin{figure}[H]
    \centering
    \includegraphics[width=1\linewidth]{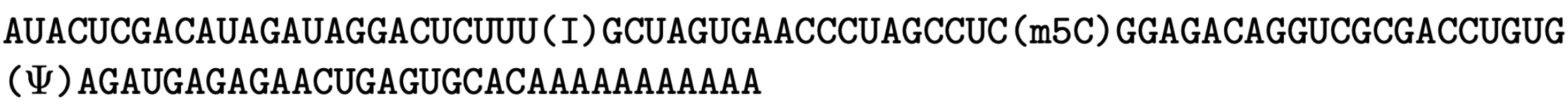}
\end{figure}
The raw nanopore direct sequencing data were obtained from the European Nucleotide Archive (accession number PRJEB44511). It contains an unmodified sample and a sample where the three modifications have been deposited. Basecalling and signal processing were performed using Dorado\footnote{\url{https://github.com/nanoporetech/dorado}}, read alignment was carried out with minimap2 \citep{li2018minimap2}, and event-level signal segmentation was obtained with Uncalled4 \citep{kovaka2025uncalled4}, yielding segmented time series representations for our analysis. 

We compared a \acrshort{ocsvm} using signatures truncated at level $N=6$ after applying two path transforms (the \emph{time augmentation} and the \emph{invisibility-reset} transforms \citep{lyons2025signature} and a \acrshort{ocsvm} trained on two standard nanopore features given by the signal duration (commonly referred to as the \emph{dwell time}) and signal mean value. We fitted the \acrshort{ocsvm} using $3\,000$ reads from the unmodified sample for each site, and computed the empirical p-values using $100\,000$ other unmodified reads. Figure~\ref{fig:RNA} shows that, at a Benjamini–Hochberg FDR level of 0.20 (with Storey correction), the signature-feature model yields consistently higher recall for all three modification types, namely inosine (I), 5-methylcytosine (m5C), and pseudouridine ($\Psi$).

\section{Conclusion}

In this work, we have framed novelty detection on path space as a hypothesis testing problem based on signature statistics. By leveraging the transportation-cost inequalities of~\citet{gasteratos2023transportation}, we established tail bounds for false positive rates that go beyond Gaussian measures to encompass laws of RDE solutions with smooth bounded vector fields, thereby enabling rigorous quantile and p-value estimates. Through the shuffle product structure, we obtained exact formulas for smooth surrogates of conditional value-at-risk (CVaR) in terms of expected signatures, which in turn motivated new one-class SVM algorithms optimising smooth CVaR objectives. We further derived lower bounds on type-$\mathrm{II}$ error under alternatives with finite first moment, leading to general power bounds in the absolutely continuous setting. Finally, our numerical experiments on synthetic anomalous diffusion data and real molecular biology data demonstrated both the validity of type-$\mathrm{I}$ error control and the practical effectiveness of signature-based testing procedures.

These results highlight the dual theoretical and algorithmic contributions of signature methods to statistical testing and novelty detection on path space. Future directions include extending the analysis to heavier-tailed drivers, developing computationally efficient estimators of expected signatures in high dimensions, and exploring applications to other domains where path-dependent anomalies play a central role.

\acks{Acknowledgements.}
This work was supported in part by EPSRC (NSFC) under Grant EP/S026347/1, in part by The Alan Turing Institute under the EPSRC grant EP/N510129/1.

\newpage 
\appendix 

\section{Signature test statistics}\label{sec:test_statistcs}

We provide more details on several commonly-used anomaly scores on path space, which are given (or can be approximated) by an element of a signature kernel RKHS $\h$. These include distance-based score functions as well as solutions of one-class SVM optimisation problems. Throughout this section, $\Hh$ will denote Hilbert spaces induced by the covariance of a measure $\mu.$ When $\mu$ is Gaussian, $\Hh$ will be the Cameron-Martin space of $\mu.$ This choice of notation is made to distinguish the latter from the signature kernel RKHS.

\subsection{Conformance score}\label{ssec:conformance_score}

Let $(\X, \|\cdot\|_{\X})$ be a normed vector space and $\mu\in\mathscr{P}(\X)$ a Borel probability measure on~$\X$.
 Based on~\cite{cochrane2020anomaly}, we define two notions of conformance of a point $x$ to a corpus $C\subset\X$:
 one based only on the topology of $\X$ and one that leverages the structure of the measure $\mu$.
 
 \begin{definition}\label{def:variance_norm}\ 
 \begin{enumerate}
 \item The \textbf{covariance} of $\mu$ induces a bilinear form on the topological dual $\X^*$ as follows:
 For all $x^*, y^*\in\X^*$,
 $$
 \mathrm{Cov}_\mu(x^*, y^*):=\int_{\X}x^*(x) y^*(x)\;d\mu(x)=\ex_\mu[x^*y^*].
 $$
 \item The \textbf{variance norm} of $x\in\X$ is given by
     $$\| x\|_{\mu}:= \sup_{\mathrm{Cov}_\mu(x^*, x^*)\leq 1} x^*(x).$$

\item For a Borel measurable set $C\in\mathscr{B}(\X)$ and a point $x\in\X$, the $\mu$-\textbf{conformance} of 
  $x$ to $C$ is given by
  \begin{equation*}
      dist_\mu(x, C) := \inf_{y\in C}\|x-y\|_{\mu},
  \end{equation*}
  where $\|\cdot\|_{\mu}$ denotes the variance norm of $\mu$ on $\X$.
  \item For a Borel measurable set $C\in\mathscr{B}(\X)$ and a point $x\in\X$, the \textbf{topological conformance} of 
  $x$ to $C$ is given by
  \begin{equation*}
      dist_{\X}(x, C) := \inf_{y\in C}\|x-y\|_{\X}.
  \end{equation*}
  \end{enumerate}
   \end{definition}
   \begin{remark}
       The topological conformance of $x$ to $C$ is nothing but the classical distance of a point $x$ to a subset $C$ of a metric space. 
   \end{remark}

\noindent The case where $\mu$ is a \textit{Gaussian} measure and $\X$ a separable Banach space is special. One can then define a linear space $\Hh$ contained in $\X$ by
$$
\Hh:=\big\{ x\in\X : \|x\|_{\mu}<\infty \big\}.
$$
For every element $h\in\Hh$ there exists a representative $h^*\in\X^*$ such that $\|h\|_{\mu}=\mathrm{Cov}_\mu(h^*, h^*)$ and moreover $\Hh$ is a Hilbert space with inner product given by 
$$\langle h_1, h_2\rangle_\Hh:=\mathrm{Cov}_{\mu}(h_1^*, h_2^*).$$
The latter in fact shows that $\|\cdot\|_{\Hh}$ coincides with the variance norm $\|\cdot\|_{\mu}$ above. 
The proof of the following result is contained in~\cite[Section 4.2]{hairer2009introduction}, 
in particular Exercise~4.38 therein. 
\begin{proposition} Let $\X$ be a separable Banach space and $\mu\in\mathscr{P}(\X)$ a Gaussian measure. Then the variance norm $\|\cdot\|_{\mu}$ is equal to the norm $\|\cdot\|_{\Hh}$ of the Cameron-Martin Hilbert space $\Hh \subset\X$ associated to $\mu$.  
\end{proposition}

\noindent The following inequality explains why the conformance defined above is appropriate for detecting anomalies in data that is assumed to come from a Gaussian distribution.

\begin{theorem}[TSB inequality]\label{thm:TSB} Let $\mu$ be a \textit{Gaussian} measure on a separable Banach space $\X$. Let $C\in\mathscr{B}(\X)$, $\widehat{c}:=\Phi^{-1}(\mu(C))$ and $\overline{\Phi}=1-\Phi$, where $\Phi$ is the distribution function of a standard Normal distribution. Then for all $r\geq 0$,
      $$\mu\bigg( x\in\X: dist_\mu(x, C)\geq r     \bigg)\leq \overline{\Phi}(r+\widehat{c}).$$  
Moreover, for $r\geq -\widehat{c}$
we have 
 $$
 \mu\bigg( x\in\X: dist_\mu(x, C)\geq r     \bigg)
\leq \exp\left\{-\frac{(r+\widehat{c})^2}{2}\right\}.
$$ 

\end{theorem}
\begin{proof}
This a straightforward reformulation of the TSB inequality in~\cite[Theorem 11.6]{friz2020course}.
In particular, by the triangle inequality, along with the fact that the distance is nonnegative, we have 
\begin{equation*}
    \begin{aligned}
        \mu\bigg( x\in\X: dist_\mu(x, C)\geq r     \bigg)&\leq \mu\bigg( x\in\X: dist_\mu(x, C)\geq r   \bigg) 
        \\&\leq  \mu\bigg( x\in\X: \forall y\in C,\; \| x-y\|_{\mu}=\| x-y\|_{\Hh}\geq r  \bigg)\\&\leq  \mu\bigg( x\in\X: x\notin C+rB_\Hh     \bigg)\leq \overline{\Phi}(\widehat{c}+r),
    \end{aligned}
\end{equation*}
where $B_\Hh$ denotes the unit ball in $\Hh$ and the last inequality follows from the TSB inequality mentioned above. 
The second assertion then follows by the elementary bound $\overline{\Phi}(z)\leq e^{-z^2/2}$ which holds for all $z\geq 0$. 
\end{proof}
 
\begin{remark}\label{remark:concentration} As mentioned in~\citep{cochrane2020anomaly}, if the corpus has probability at least~$\frac{1}{2}$ 
($\mu(C)\geq \frac{1}{2})$, then $\widehat{c}\geq \Phi^{-1}(\frac{1}{2})=0$.  Therefore, the probability that the conformance of a (random) sample point $X$ to $C$ exceeds a threshold $r$ is  bounded above by $e^{-r^2/2}$. In fact, in this case, the second moment of the $\mu$-conformance can be bounded uniformly over~$\widehat{c}$:

\begin{equation*} \begin{aligned}
\ex_\mu\left[dist^2_{\mu}(X, C)\right]
 &= 2\int_{0}^{\infty}r\cdot\mu\big(x: dist_{\mu}(x,C)>r\big)dr\\&
 \leq 2\int _{0}^{\infty}r\overline{\Phi}(\widehat{c}+r)dr\leq
 2\int _{0}^{\infty}r\overline{\Phi}(r)dr=\frac{1}{2}.
     \end{aligned}
 \end{equation*}

 On the other hand, if $\mu(C)<\frac{1}{2}$, then $\widehat{c}<0$. In this case, the TSB inequality provides a $\widehat{c}$-dependent upper bound for the second moment of the $\mu$-conformance:
 \begin{equation*}
\begin{aligned}
 \ex_\mu\left[dist^2_{\mu}(X, C)\right]
  &= 2\int_{0}^{\infty}r\cdot\mu(x: dist_{\mu}(x,C)>r)dr\\&
 \leq 2\int _{0}^{-\widehat{c}}r\overline{\Phi}(\widehat{c}+r)dr+2\int _{-\widehat{c}}^{\infty}r\overline{\Phi}(\widehat{c}+r)dr
 \\&\leq 
 2\int _{0}^{-\widehat{c}}rdr+2\int _{0}^{\infty}(z-\widehat{c})\overline{\Phi}(z)dz
\\&=\widehat{c}^2+2\int _{0}^{\infty}\bigg(z\overline{\Phi}(z)-\widehat{c}\overline{\Phi}(z)\bigg)dz
=\widehat{c}^2+\frac{1}{2}-\widehat{c}\sqrt{\frac{2}{\pi}},
     \end{aligned}
 \end{equation*}
where we bounded $\bar{\Phi}$ from above by $1.$  This case is particularly relevant when $\X$ is a high-dimensional vector space: as the dimension of $\X$ grows, the volume $\mu(C)$ of the corpus decreases. 
Moreover, as $\mu(C)$ tends to zero, then $\widehat{c}$ decreases to $-\infty$ and the tails of the conformance score are upper bounded by those of a Gaussian with mean converging to $\infty;$ one then expects that the conformance concentrates to increasingly larger values as the dimension of $\X$ grows.
 Thus, an increasingly larger portion of the data will achieve higher conformance scores. Anomalies should then be classified by adding $\widehat{c}$ to the conformance score.
\end{remark}

\subsection{One-class support vector machines}\label{sec:ocsvm_appendix}

Here, we review \glspl{ocsvm} through the lens of the conditional value-at-risk (CVaR), a concept originating in financial risk management and stochastic programming, also known as the superquantile, expected shortfall, or average value-at-risk. The CVaR of a random variable $Z$ summarises the average behaviour of the tail of its distribution. \cite{takeda2008nu} showed that extended $\nu$-support vector classifiers (E$\nu$-SVC) introduced in~\citep{perez2003extension} can be interpreted as minimising a CVaR. Subsequently,~\cite{tsyurmasto2014value} showed that \glspl{ocsvm} can be also be interpreted as CVaR minimisation. More precisely, \glspl{ocsvm} seek $w\in\h$ that minimises the regularised CVaR
$$
w^*=\arg\min_{w\in \h} \bigg\{\mathrm{CVaR}_{\alpha}(\langle w, \varphi(X)\rangle_\h)+\frac{1}{2} \|w\|^2_{\h}\bigg\}.
$$
Then, for any $\alpha \in [0,1]$, the decision function is defined by $\mathds{1}_{\Omega^{\mathrm{CVaR}}_{\alpha}}(\cdot)$ where 
\begin{align}\label{eq:region_cvar}
\Omega^{\mathrm{CVaR}}_\alpha := \{x\in\X:\langle w^*, \varphi(x)\rangle_\h\leq \text{VaR}_\alpha(\langle w^*, \varphi(X)\rangle_\h)\}. 
\end{align}
By the variational formulation of the CVaR, we have
\begin{align*}
w^* = \arg\min_{w\in \h}\left\{
 \min_{\eta\in \R}
\left\{
\eta + \frac{\ex_\mu[[\langle w,\varphi(X)\rangle_\h-\eta]^+]}{1-\alpha}\right\}
+ \frac{1}{2}\|w\|_{\h}^2
\right\}.
\end{align*}
The change of variables $\rho=-\eta$, $v=-w$, $\gamma=1-\alpha$ yields
\begin{align*}
v^* = \arg\min_{v\in \h}
\left\{
\frac{1}{2}\|v\|_{\h}^2
+\min_{\rho\in \R}
\left\{
\frac{1}{\gamma} \ex_\mu[[-\langle v,\varphi(X)\rangle_\h+\rho]^+] -\rho\right\}
\right\}.
\end{align*}
When the measure $\mu$ is unknown, but a sample $x_1, \ldots, x_n\sim\mu$ is available, an empirical estimator \citep{wang2014robust} may be formed 
\begin{align}\label{eq:ocsvm_unconstrained}
(\hat{v}^*, \hat{\rho}^*)=\arg\min_{v\in \h,\rho\in \R} 
\left\{
\frac{1}{2}\|v\|_{\h}^2
+\left(
\frac{1}{\gamma n}\sum_{i=1}^{n}[\rho-\langle v, \varphi(x_i)\rangle_\h]^+ -\rho
\right)
\right\},
\end{align}
which is the unconstrained version of the quadratic programming problem~\citep{xiao2017ramp, norton2017soft}
\begin{align*}
   &\min_{v\in \h,\rho\in \R,\xi \in \R^n}\left\{
   \frac{1}{2}\|v\|_{\h}^2
   + \left(\frac{1}{\gamma n}\sum_{i=1}^{n}\xi_i-\rho\right)\right\},\\
    &\text{subject to } \langle v, \varphi(x_i)\rangle_\h\geq \rho -\xi_i
    \text{ and } \xi_i\geq 0,\text{ for all }i=1,\ldots,n
\end{align*}
whose solution is such that all training points are on one side of an hyperplane in the feature space with maximum margin. The empirical \gls{ocsvm} seeks a decision function of the form $\mathds{1}_{\Omega^{\mathrm{ocsvm}}_\gamma}(\cdot)$ (takes the value $1$ for normal data and $0$ otherwise where
\begin{align*}
\Omega^{\mathrm{ocsvm}}_\gamma=\{x\in\X:\langle \hat{v}^*, \varphi(x)\rangle_\h\geq \hat{\rho}^*\}.
\end{align*}

\section{Conformance score for non-Gaussian streamed data}\label{Sec:Donsker}
In practice, the assumption that the underlying sample comes from a Gaussian is restrictive. In the sequel we shall investigate under what assumptions on the distribution of the underlying data one can have similar types of estimates for large values of the conformance. The goal is to derive such guarantees in the setting of streamed data i.e. we assume that observations are points on the space $C([0,T];\R^d)$.

\begin{definition} The set of streams of $d$-dimensional data is defined as 
$$
\mathcal{S}(\R^d)
:=\bigg\{ \mathbf{x}=(x_1,\dots, x_n)\in(\R^d)^n, n\in\N\bigg\}.
$$
\end{definition}
\noindent Given $k\in\N$ and a finite set $\{\mathbf{x}^i\}_{i=1}^{k}=\{(x^i_1, \dots, x^{i}_{n_i})\}_{i=1}^{k}\subset \mathcal{S}(\R^d)$ of $k$ streams, 
let $\ell(\mathbf{x}^i):=n_i$ denote the length of the $i$-th stream and
\begin{equation*}
    \label{eq:Ldef}
    L:=\max_{i=1,\dots, k}\ell(\mathbf{x}^i)
\end{equation*}
be the maximum stream length in the family 
$\{\mathbf{x}^i\}_{i=1}^{k}$. 
For each stream $\mathbf{x}^i$ with $\ell(\mathbf{x}^i)<L$, consider the stream
$$
\widetilde{\mathbf{x}}^i=\left(x^i_1, x^i_2,\dots, x^i_{n_i},x^i_{n_i}, \dots, x^i_{n_i}\right)
$$
of length $L$ augmented with the last observation in the remaining spots; note that the signature is invariant under this augmentation. Hence, the set $\{\widetilde{\mathbf{x}}^i\}_{i=1}^{k}$ consists of streams of equal length~$L$.

\begin{assumption}\label{Assumption:iid}
The streams $\{\widetilde{\mathbf{x}}^i\}_{i=1}^{k}$ are $k$ realisations of an i.i.d.
sample $X=(X^1, \dots, X^L)$ from an underlying distribution $\gamma\in\mathscr{P}(\R^d)$ with mean zero and a finite second moment i.e. for each $i=1, \dots, k$, $X(\omega_i)=\widetilde{\mathbf{x}}^i$ for some $\omega_i\in\R^d$. 
Moreover the covariance matrix $\Sigma^\gamma=\{\Sigma^\gamma_{i,j}\}_{i,j=1}^{d}:=\{ \ex_\gamma[X^{1}_{i}X^{1}_{j}]\}_{i,j=1}^{d}$ is invertible. A fortiori, we have 
$$
\sigma_\gamma^2:=tr\Sigma^\gamma=\int_{\R^d} |x|_2^2\;\gamma(dx)<\infty.
$$ 
\end{assumption}

At this point, we shall transition to the view of data points as vectors on the path space $C([0,1]; \R^d)$. Indeed, given an i.i.d. sample $X$ from $\mu$ as above, we consider a continuous random path $W^L$ defined via linear interpolation as
\begin{equation}\label{eq:Donskerapprox}
W^L_t:=\frac{1}{ \sqrt{L}}\sum_{k=1}^{[Lt]}X^{k}+\frac{(Lt-[Lt])}{\sqrt{L}}X^{[Lt]+1},
\qquad\text{for all }
t\in[0,1].
\end{equation} 
Then, Condition~\ref{Assumption:iid} placed above implies the following:
\textit{
For each $i=1,\ldots, k$, the continuous path}
$$
\phi^{i, L}(t) :=    \frac{1}{\sqrt{L}}\sum_{k=1}^{[Lt]}x^i_{k}+\frac{(Lt-[Lt])}{\sqrt{L}}x^{i}_{[Lt]+1},\qquad\text{for all }
t\in[0,1],
$$
\textit{associated to the} $i$-\textit{th stream }$\mathbf{\tilde{x}^i}$, \textit{is a realisation of the random continuous path} $W^L$.

Let us denote by $\mu^L$ the law of $W^L$ on $C([0,1]; \R^d)$. By the multidimensional version of Donsker's invariance principle~\cite[Chapter XIII, Theorem 1.9]{revuz2013continuous}, we know that as~$L$ tends to infinity,
$\mu^L$ converges weakly in $C([0,1]; \R^d)$ to the law of a $d$-dimensional Wiener process $W$ with covariance 
\begin{equation}\label{eq:BM}
    \ex[W^i_t W^j_s]= \Sigma^\gamma_{i,j}(t\wedge s),
    \quad\text{for all } t,s\in[0,1], i,j=1,\dots, d.
\end{equation}

The following provides a guarantee for the topological conformance of a sample drawn from the approximating distribution:

\begin{proposition}\label{prop:nonGaussianTSB} Let $\X$ be a separable Banach space, $C$ a Borel subset of $\X$ and $\{\mu^L\}_{L\in\N}\in \mathscr{P}(\X)$ a sequence of probability measures converging weakly to a Gaussian measure $\mu \in \mathscr{P}(\X)$. If  $K>0$ is the smallest constant such that  $\|x\|_{\X}\leq K\| x\|_{\mu}$
for all $x\in\X$
then, for all $r>0$,
\begin{equation*}
    \begin{aligned}       \limsup_{L\to\infty}\mu^L\bigg( x\in\X: dist_\X(x, C)\geq r     \bigg)
 \leq \overline{\Phi}(\widehat{c}+r/K)\leq 
 \exp\left\{-\frac{(\widehat{c}+r/K)^2}{2}\right\},
    \end{aligned}
\end{equation*}
where $\widehat{c}:=\Phi^{-1}(\mu(C))$. Moreover, if~$C$ is closed, we can replace~$\widehat{c}$ by $\Phi^{-1}(\limsup_{L\to\infty}\mu^L(C))$.
\end{proposition}

\begin{proof} 
Since $\X\ni x\mapsto dist_{\X}(x, C)$ is Lipschitz continuous, then $A:=\{x: dist_{\X}(x, C)\geq r\}$ is a closed subset of $\X$. 
By the Portmanteau theorem for weak convergence of measures we know that $\limsup\limits_{L\to\infty}\mu^L(A)\leq \mu(A)$. Putting these facts together we derive
\begin{align*}
\limsup_{L\to\infty}\mu^L\bigg( x\in\X: dist_\X(x, C)\geq r     \bigg)&\leq  \mu\bigg( x\in\X: dist_\X(x, C)\geq r     \bigg)\\&\leq
        \mu\bigg( x\in\X: dist_\gamma(x, C)\geq r/K     \bigg)
 \\&\leq \overline{\Phi}(\widehat{c}+r/K).
\end{align*}
The last two inequalities are obtained using $dist_{\X}(x, C)\leq K dist_{\mu}(x, C)$
and the TSB inequality (Theorem~\ref{thm:TSB}). The last assertion follows again from the Portmanteau theorem and the monotonicity of $\Phi^{-1}$ since 
$\overline{\Phi}(\widehat{c}+r/K)\leq \overline{\Phi}(C+r/K)$
and
$$  C:=\Phi^{-1}\bigg(\limsup_{L\to\infty}\mu^L(C)\bigg)\leq \Phi^{-1}\bigg(\mu(C)\bigg) =\widehat{c}.
$$
\end{proof}

\begin{remark} [Dimensionality considerations: the uncorrelated case]
The constant~$K$ in the theorem above does not depend on the dimension $d$ of the data in the case where the covariance matrix $\Sigma^\gamma$ is the identity.
For example, if $\mu$ is the law of a standard $d$-dimensional Wiener process on $\X=C([0,1];\R^d)$, the corresponding Cameron-Martin space is given by $$
\Hh:=\bigg\{ x\in L^2([0,1];\R^d): x(0)=0, \dot{x}\in L^2([0,1];\R^d)\bigg\},
$$
and for $x\in\Hh$,

$$
\|x\|_{\X} :=
\sup_{t\in[0, 1]}| x(t)|_{\R^d}=\sup_{t\in[0, 1] }\bigg|\int_{0}^{1} \dot{x}(t)dt\bigg|_{\R^d}\leq \bigg(\int_{0}^{1} |\dot{x}(t)|^2_{\R^d}dt \bigg)^{1/2}=\|x\|_{\Hh}\equiv\|x\|_{\mu}. $$ 
Hence $K$ can be taken to be equal to $1$. However, what \textit{does} (implicitly) depend on dimension is the value of the topological conformance $dist_{\X}(x, C)$ because the norm of~$\X$ depends on~$d$.
\end{remark}
The following straightforward consequence of Proposition~\ref{prop:nonGaussianTSB} provides asymptotic conformance estimates in the particular case where $\mu$ is the law of a $d$-dimensional Wiener process with (invertible) covariance matrix $\Sigma^\gamma$.
\begin{corollary}\label{cor:donsker} Let $\mu$ be the law of $W$ in~\eqref{eq:BM} and~$\mu^L$ the law of the approximating path~$W^L$ in~\eqref{eq:Donskerapprox}  on $C([0,1]; \R^d)$. 
Under Condition~\ref{Assumption:iid}, for any Borel measurable subset $C\subset C([0,1];\R^d)$ and any $r>0$,
\begin{equation}\label{eq:Sigmabnd}
\begin{aligned}
   \limsup_{L\to\infty}\mu^L\bigg( x\in C([0,1];\R^d): & dist_\X(x, C)\geq r     \bigg)
 \leq \overline{\Phi}\left(\widehat{c}+r\|\Sigma^\gamma\|^{\frac{1}{2}}_{op}\right)\\&\leq 
 \exp\left\{-\frac{1}{2}\left(\widehat{c}+r\|\Sigma^\gamma\|^{\frac{1}{2}}_{op}\right)^2\right\}, 
\end{aligned}
\end{equation}
where $\widehat{c}:=\Phi^{-1}(\mu(C))$ and $\|\cdot\|_{op}$ a matrix operator norm compatible with the norm on $\R^d$ i.e. for all $x\in\R^d$, $|\Sigma^\gamma x|_{\R^d}\leq \|\Sigma^\gamma\|_{op}|x|_{\R^d}$. Moreover, if the set $C$ is closed, we can replace the constant $\widehat{c}$ by $\Phi^{-1}(\limsup_{L\to\infty}\mu^L(C))$.
\end{corollary}

\begin{proof} By the multivariate version of Donsker's theorem~\cite[Chapter~XIII, Theorem~1.9]{revuz2013continuous},
$\mu^L$ converges weakly to $\mu$ as $L\to\infty$ in the topology of $\X=C([0,1];\R^d)$. The variance/Cameron-Martin norm of $\mu$ is given by 
$$
\|x\|^2_{\mu}=\int_{0}^{1}\big|(\Sigma^\gamma)^{1/2} \dot{x}(t)\big|_{\R^d}^2dt,
$$
and 
$$\|x\|_{\X}\leq \|\Sigma^\gamma\|^{-1/2}_{op} \bigg(\int_{0}^{1}\big|(\Sigma^\gamma)^{1/2}  \dot{x}(t)\big|_{\R^d}^2dt\bigg)^{1/2}=\|\Sigma^\gamma\|^{-1/2}_{op} \|x\|_{\mu}.
$$
This implies that the constant $K$ in Proposition~\ref{prop:nonGaussianTSB} can be taken equal to $\|\Sigma^\gamma\|^{-1/2}_{op}$ for any matrix operator norm.
\end{proof}

Corollary~\ref{cor:donsker} suggests that large values of the topological conformance 
$d_{\X}(x; C)$, under the measure $\mu^L$, scale according to the covariance norm $\|\Sigma^\gamma\|_{op}$. As the dimension grows it is clear that this quantity increases. For example, if $\R^d$ is equipped with the Euclidean (2-)norm, then $\|\Sigma^\gamma\|^2_{op}=\|\Sigma^\gamma\|^2_2=\max_{k=1,\dots, d}\lambda_k^2$, where $\lambda_k$ are the eigenvalues of the matrix~$\Sigma^\gamma(\Sigma^\gamma)^T$. Thus, in this case $\|\Sigma^\gamma\|_{op}$ is non-decreasing as a function of $d$.

In view of these facts, a sensible conformance measure of a stream $x\in C([0,1];\R^d)$, discretely sampled from an underlying distribution $\gamma\in\mathscr{P}(\R^d)$ with mean $0$ and covariance $\Sigma^\gamma$,  from a corpus $C\subset C([0,1];\R^d)$ would be the\textbf{ variance-adjusted topological conformance}: 
\begin{equation}\label{eq:Ddef}
    D_{\X, \gamma}(x; C):=\|\Sigma^\gamma\|^{1/2}_{op}dist_{\X}(x, C).
\end{equation}
Indeed, plugging the latter into~\eqref{eq:Sigmabnd} yields
\begin{align*}
\limsup_{L\to\infty}\mu^{L}&\bigg( x\in  C([0,1];\R^d):  D_{\X, \gamma}(x; C)\geq r\bigg)\\&=\limsup_{L\to\infty}\mu^{L}\bigg( x\in  C([0,1];\R^d):  dist_{\X}(x; C) \geq \frac{r}{\|\Sigma^\gamma\|^{1/2}_{op}}\bigg)
\leq \exp\left\{-\frac{(\widehat{c}+r)^2}{2}\right\}.
\end{align*}
The importance of this estimate lies in the fact that probabilities of high conformance scores are not affected by dimensionality as long as the scores are appropriately adjusted.


\newpage



\section{Experimental Details}

The raw nanopore direct sequencing data were obtained from the European Nucleotide Archive (accession number PRJEB44511). Specifically we downloaded the modified and unmodified samples with respective experiment accession numbers ERX6983963 and ERX6983965. We converted the samples into the pod5 format using the command
\begin{verbatim}
pod5 convert fast5 <path_to_fast5>/*.fast5 --output Oligo_<Mod_or_Unmod>.pod5
\end{verbatim}Using the (.fa) reference sequence 
\begin{verbatim}
>control
ATACTCGACATAGATAGGACTCTTTAGCTAGTGAACCCTAGCCTCCGGAGACAGGTCGCGACCTGTGTAGATGAGAGAAC
TGAGTGCACAAAAAAAAAAA
\end{verbatim}
we basecalled the sample using Dorado (version \texttt{0.9.1}) using the \texttt{rna002\_70bps\_hac@v3} model using the command
\begin{verbatim}
dorado basecaller <path_to_dorado_model>/rna002_70bps_hac@v3 sample.pod5 \
--reference ref.fa  --emit-moves --emit-sam > basecalled_sample.bam
\end{verbatim}
and event-level signal segmentation was obtained with Uncalled4, running the command 
\begin{verbatim}
uncalled4 align --ref ref.fa --reads sample.pod5 --bam-in basecalled_sample.bam \
--eventalign-out --flowcell FLO-MIN106 --kit SQK-RNA002 --min-aln-length 5 \
--eventalign-flags print-read-names,signal-index,samples > aligned_sample.txt
\end{verbatim}
For the modified sample, we had to additionally run the following commands before signal alignment with Uncalled4
\begin{verbatim}
samtools fastq -@8 -T "mv,ts,pi,sp,ns"  basecalled_sample.bam > sample.moves.fastq
minimap2 -y -ax map-ont -k9 -w1 --secondary=no -t8 ref.fa samples.moves.fastq \
| samtools sort -@8 -o sample.mm2.bam
\end{verbatim}

\bibliography{references}

\end{document}